\documentclass[ 10 pt, conference]{ieeeconf}  % Comment this line out
\IEEEoverridecommandlockouts                              % This command is only
\usepackage{float}
\usepackage[hidelinks]{hyperref}

\usepackage{subcaption}

\newif\ifautomatica
\automaticafalse
\usepackage{amsfonts}
\usepackage{amsmath}
\usepackage{amssymb}

\let\theoremstyle\relax

\usepackage{amsthm}

\usepackage{mathtools}
\usepackage{stmaryrd}

\usepackage{xcolor}

\let\parens\undefined
\newcommand{\parens}[1]{{\left(#1\right)}}

\let\braces\undefined
\newcommand{\braces}[1]{{\left\{#1\right\}}}

\let\abs\undefined
\newcommand{\abs}[1]{{\lvert#1\rvert}}

\let\norm\undefined
\newcommand{\norm}[1]{{\lVert#1\rVert}}

\let\mrm\undefined
\newcommand{\mrm}[1]{{\mathrm{#1}}}

\let\mbb\undefined
\newcommand{\mbb}[1]{{\mathbb{#1}}}

\let\mbf\undefined
\newcommand{\mbf}[1]{{\mathbf{#1}}}

\let\mcal\undefined
\newcommand{\mcal}[1]{{\mathcal{#1}}}

\let\msf\undefined
\newcommand{\msf}[1]{{\mathsf{#1}}}

\theoremstyle{plain}
\newtheorem{theorem}{Theorem}
\newtheorem{lemma}{Lemma}
\newtheorem{remark}{Remark}
\title{\LARGE \bf
Chordal Sparsity for SDP-based Neural Network Verification
}

\author{Anton Xue${}^\star$, Lars Lindemann${}^\dagger$, Rajeev Alur ${}^\star$% <-this % stops a space
    \thanks{${}^\star$ Anton Xue and Rajeev Alur are with the Department of Computer and Information Science, University of Pennsylvania, Philadelphia, PA, USA.
    {\tt\small \{antonxue,alur\}@seas.upenn.edu}
    }
    \thanks{${}^\dagger$ Lars Lindemann is with the Thomas Lord Department of Computer Science, University of Southern California, Los Angeles, CA, USA.
    {\tt\small llindema@usc.edu}
    }
}

\begin{document}

\maketitle

\thispagestyle{plain}
\pagestyle{plain}

\begin{abstract}

Neural networks are central to many emerging technologies, but verifying their correctness remains a major challenge.
It is known that network outputs can be sensitive and fragile to even small input perturbations, thereby increasing the risk of unpredictable and undesirable behavior.
Fast and accurate verification of neural networks is therefore critical to their widespread adoption, and in recent years, various methods have been developed as a response to this problem.
In this paper, we focus on improving semidefinite programming (SDP) based techniques for neural network verification.
Such techniques offer the power of expressing complex geometric constraints while retaining a convex problem formulation, but scalability remains a major issue in practice.
Our starting point is the DeepSDP framework proposed by Fazlyab et al., which uses quadratic constraints to abstract the verification problem into a large-scale SDP.
However, solving this SDP quickly becomes intractable when the network grows.
Our key observation is that by leveraging \emph{chordal sparsity}, we can decompose the primary computational bottleneck of DeepSDP --- a large linear matrix inequality (LMI) --- into an equivalent collection of smaller LMIs.
We call our chordally sparse optimization program \emph{Chordal-DeepSDP} and prove that its construction is identically expressive as that of DeepSDP.
Moreover, we show that additional analysis of Chordal-DeepSDP allows us to further rewrite its collection of LMIs in a second level of decomposition that we call \emph{Chordal-DeepSDP-2} --- which results in another significant computational gain.
Finally, we provide numerical experiments on real networks of learned cart-pole dynamics, showcasing the computational advantage of Chordal-DeepSDP and Chordal-DeepSDP-2 over DeepSDP.

\end{abstract}

\section{Introduction}
% https://arxiv.org/abs/2206.07811

% https://arxiv.org/pdf/2102.12981.pdf

% The success of neural networks is well documented in the literature for various applications, e.g., in  Go \cite{silver2018general}, in handwritten character recognition \cite{lecun1998gradient}, and in autonomous driving, \cite{dosovitskiy2017carla}.

The success of neural networks is well-documented in dynamical systems and control~\cite{yeung2019learning,dosovitskiy2017carla}.
Neural networks are, however, notoriously opaque.
The key challenges in analyzing their behavior are two-fold: first, one must appropriately model the nonlinear activation functions; second, the neural network may be very large.
In addition, neural networks are often sensitive to input perturbations~ \cite{goodfellow2014explaining}, meaning that test-driven methods may insufficiently cover the input domain.
Hence, It is often unclear what a neural network learns and how specific desirable properties can be verified.
This is of particular concern in safety-critical applications like autonomous driving, where the lack of formal verification guarantees poses a serious barrier to adoption.
We focus on \emph{safety verification}: given a neural network \(f:\mathbb{R}^{n_1}\to\mathbb{R}^m\) and a specification \(\mcal{S} \subseteq \mcal{X} \times \mbb{R}^m\) where \(\mcal{X} \subseteq \mbb{R}^{n_1}\) is the input domain, does \(f\) satisfy \(\mcal{S}\)?
In other words, is it the case that \((x, f(x)) \in \mcal{S}\) for all \(x \in \mcal{X}?\)
% \begin{align*}
%     (x, f(x)) \in \mcal{S} \enskip \text{for all} \enskip x \in \mcal{X}?
% \end{align*}

In this work, we investigate the properties of semidefinite programming (SDP)-based methods for neural network verification.
SDPs are interesting because they can naturally encode tight relaxations of quadratic relations, allowing for strictly more expressive specifications and abstractions beyond commonly used linear ones~\cite{katz2019marabou,muller2022prima,tjeng2017evaluating}.
Our starting point is the DeepSDP verification framework presented in \cite{fazlyab2020safety}, focusing on improving scalability.
DeepSDP uses quadratic constraints to abstract nonlinear activation functions, yielding a convex relaxation of the safety verification problem as a large SDP.
When the neural network grows larger, however, solving DeepSDP quickly becomes intractable --- a common issue in SDP-based methods --- thus motivating our work.

We observe that although DeepSDP~\cite{fazlyab2020safety} itself does not exploit any sparsity to improve scalability, it was recently noted in \cite{newton2021exploiting} that a particular instantiation of DeepSDP for neural networks with ReLU action functions admits in its LMI constraint \emph{chordal sparsity}~\cite{agler1988positive,vandenberghe2015chordal,zheng2019chordal} --- the key computational bottleneck.
This insight allows the authors of \cite{newton2021exploiting} to decompose and efficiently solve this sparsified version of DeepSDP.
This approach improves scalability but omits two features critical to DeepSDP.
First, the authors make a restrictive assumption on the admissible safety specification to simplify sparsity analysis.
Concretely, the specifications of \cite{newton2021exploiting} no longer allow for arbitrary quadratic constraints involving input-output couplings, such as the \(L_2\) gain \(\kappa\) of \(f\) posed as \(\norm{f(x)}_2 \leq \kappa \norm{x}_2\).
We show that one does not need to make this restriction and develop our approach without assuming any restrictions on the original safety specifications of DeepSDP.
Second, the authors of ~\cite{newton2021exploiting} assume only the use of \emph{sector-bounded} abstractions for the activation functions and do not consider how other abstractions may be introduced without altering the sparsity pattern.
We show how constraints in the form of \emph{adjacent-layer abstractions} may be freely added without affecting our decomposition method.
Notably, the authors of~\cite{newton2021exploiting} only conjecture chordal sparsity in DeepSDP, while this paper formally proves it.

In this work, we first present a decomposition that is equivalently expressive to DeepSDP, in particular allowing for arbitrary quadratic input-output specifications --- a distinguishing advantage of SDP-based techniques.
Then, we show how further analysis yields a second level of chordal decomposition that leads to a significant performance gain, still without any restrictions on the safety specification.
Furthermore, we show how additional abstractions of a certain form may be introduced without altering the chordal decomposition structure.
Improving the scalability of DeepSDP is relevant for other SDP-based methods, which continue to use many ideas pioneered by DeepSDP.
Our \textbf{contributions} are as follows:

\begin{itemize}
    % {\color{blue}
    % \item We show that DeepSDP admits a chordally sparse decomposition, allowing us to formulate an equivalent semidefinite program that we call Chordal-DeepSDP.
    % The primary benefit of this approach is scalability, as the main computational bottleneck of DeepSDP --- a large LMI --- is now split into an equivalent collection of smaller LMIs.
    % Moreover, DeepSDP and Chordal-DeepSDP are equivalent in that both compute the same optimal value, implying zero accuracy loss.
    % }

    \item We propose the Chordal-DeepSDP method for SDP-based verification of neural networks with general activation functions.
    We prove that all instances of DeepSDP admit a chordally sparse decomposition regardless of safety constraints.
    The primary benefit is scalability, as the main computational bottleneck of DeepSDP, a large LMI, may be split into an equivalent collection of smaller LMIs with zero accuracy loss.

    \item
    We present techniques for bolstering the accuracy and scalability of Chordal-DeepSDP.
    We show how additional constraints may be included, particularly \emph{adjacent-layer abstractions},
    without changing the chordal decomposition structure.
    Moreover, we exploit a second level of chordal decomposition to derive yet another efficient SDP that we call Chordal-DeepSDP-2.
    
    % Moreover, we observe that a second level of chordal decomposition is possible, allowing us to derive a further decomposed semidefinite program that we call Chordal-DeepSDP-2.
    
    \item 
    We provide numerical experiments to demonstrate the scalability of Chordal-DeepSDP and Chordal-DeepSDP-2 on real networks.
    Using learned cart-pole dynamics, we show that our chordal decompositions significantly out-scale DeepSDP in performance without losing accuracy.
    Our implementation is open-sourced at: \url{https://github.com/AntonXue/nn-sdp}.
    
\end{itemize}

\subsection{Related Work}
% https://www.aaai.org/AAAI22Papers/AAAI-4268.LanJ.pdf

% check also https://scholar.google.com/citations?hl=de&user=xdWZkEEAAAAJ&view_op=list_works&sortby=pubdate

The safety verification of neural networks has found broad interest over the past years, and there have been various efforts to address the problem~\cite{bak2021second,liu2019algorithms}.
The first breakthrough was with Reluplex~\cite{katz2017reluplex}, a specialized satisfiability modulo theories (SMT) based verification method that interleaves the simplex procedure with boolean satisfiability solving.
While Reluplex only applies to feedforward neural networks with ReLU activation functions, the technique was extended to fully connected and convolutional neural networks with arbitrary piecewise-linear activation functions with Marabou~\cite{katz2019marabou}.
Other works in SMT-based verification include~\cite{song2021qnnverifier,sena2021verifying}.
Related to SMT-solving has been the use of mixed-integer programming such as in~\cite{lomuscio2017approach,tjeng2017evaluating}.
These methods give accurate results for various types of neural networks and are, in fact, complete for ReLU activation functions, but they are fundamentally hard-to-solve combinatorial problems.

Another way of approaching safety verification is by reformulating it as a reachability problem~\cite{ivanov2019verisig,tran2020nnv,xiang2018output}, which has been of particular interest for the safety verification of closed-loop dynamical systems~\cite{everett2021neural}.
While heuristics for reachability problems over dynamical systems exist, these methods are typically computationally expensive.
Safety verification methods based on abstract interpretations~\cite{gehr2018ai2,tran2020nnv,muller2022prima} work by soundly over-approximating the input set into a convenient abstract domain, e.g., polytopes or zonotopes, and then propagating the abstracted representation through the network.
Then, properties verified on the abstract output domain imply that the same properties hold on the original system.
These approaches generally offer good scalability, but one must carefully choose appropriate abstract domains to preserve accuracy.

Broadly related to this work is the idea to formulate neural network verification as a convex optimization problem.
For scalability, linear approximations have been presented in  \cite{chen2021deepsplit,wong2018provable} where particularly \cite{chen2021deepsplit} is amendable to parallelization via ADMM.
% ~\cite{boyd2011distributed}.
Techniques for interval analysis like \(\beta\)-CROWN~\cite{wang2021beta} are also GPU-parallelizable.
Further related are convex optimization-based techniques~\cite{dvijotham2018dual}, which commonly use convex over-approximations of the activation functions to formulate a convex verification problem.

In summary, we study the scalability of semidefinite programming methods with a focus on DeepSDP~\cite{fazlyab2020safety,fazlyab2021introduction}.
The authors in~\cite{raghunathan2018semidefinite} present an early work in this direction, encoding ReLU activations via semidefinite constraints.
This was improved upon in DeepSDP~\cite{fazlyab2020safety}, which can handle arbitrary activations --- so long as they satisfy a quadratic \emph{sector-bounded} condition.
The work of \cite{newton2021exploiting} recognized that certain restricted parameterizations of DeepSDP admit sparsity structures amenable to chordal decomposition, but such formulations tend to be conservative and limited in their expressiveness.
The notion of chordal sparsity has proven helpful in decomposing large-scale optimization problems and is extensively studied \cite{agler1988positive,vandenberghe2015chordal,zheng2019chordal} with application in \cite{chen2020chordal,ihlenfeld2022faster,mason2014chordal}.
Recent works on SDP-based neural network verification have examined finding tighter abstractions, especially for networks with ReLU activations~\cite{batten2021efficient,lan2021tight}.
Also related are works based on polynomial optimization~\cite{brown2022unified,chen2020semialgebraic,newton2021neural,newton2022sparse}, which similarly use semidefinite programming as the underlying computational technique.
Many of our sparsity analysis techniques are based on the work of~\cite{xue2022chordal}, which studies chordal decompositions of LipSDP~\cite{fazlyab2019efficient}.

\section{Background and Problem Formulation}
\label{sec:background}

We consider feedforward neural networks \(f:\mbb{R}^{n_1}\to \mbb{R}^m\) with \(K \ge 2\) layers, i.e., \(K - 1\) hidden layers. For an input \(x_1 \in \mbb{R}^{n_1}\),  the output of the neural network is recursively computed for \(k = 1, \ldots, K-1\) as
\begin{align*}
    f(x_1) \coloneqq W_K x_K + b_K,
    \quad x_{k+1} \coloneqq \phi (W_k x_k + b_k),
\end{align*}
where \(W_k\) and \(b_k\) are the weight matrices and bias vectors of the \(k\)th layer that are of appropriate dimensions.
The dimensions of \(x_1, \ldots, x_K\) are \(n_1, \ldots, n_K > 0\).
The function \(\phi(u)\coloneqq \text{vcat} (\varphi(u_1), \varphi(u_2) \hdots) \in \mbb{R}^n\) is the activation function \(\varphi : \mbb{R} \to \mbb{R}\) applied element-wise to the input \(u = \text{vcat}(u_1, u_2 \hdots) \in \mbb{R}^n\).
Throughout the paper, we assume the same activation function is used for all neurons.
Let \(N \coloneqq n_1 + \cdots + n_K\) and also define the vectors
\begin{align*}
    \mbf{x} \coloneqq \text{vcat}(x_1, \hdots, x_K) \in \mbb{R}^{N},
    \enskip \mbf{z} \coloneqq \text{vcat}(\mbf{x}, 1) \in \mbb{R}^{N+1},
\end{align*}
with projections \(E_a \mbf{z} = 1\), \(E_k \mbf{z} = x_k\) for \(k = 1, \ldots, K\).

%Given an input set $X\subseteq \mbb{R}^{n_1}$ and a safe output set $Y\subseteq \mbb{R}^m$, the safety verification problem that we address in this paper is now to check if $f(X)\subseteq Y$.

\subsection{Neural Network Verification with DeepSDP}

The authors of \cite{fazlyab2020safety} present the DeepSDP framework for verifying whether a network \(f\) satisfies a specification \(\mcal{S}\).
The main idea is to use \emph{quadratic constraints} (QCs) to abstract the input set \(\mcal{X}\), the activation function \(\phi\), and the safety specification \(\mcal{S}\).
We then take the input QC~\eqref{eq:input-qc}, the activation QC~\eqref{eq:deepsdp-Q}, and safety QC~\eqref{eq:safety} to set up the safety verification problem as a large \emph{semidefinite program} that we call DeepSDP~\eqref{eq:deepsdp} whose satisfiability implies that the network is safe.

% The main idea is to use \emph{quadratic constraints} (QCs) to abstract the input set \(\mcal{X}\), the activation functions \(\phi\), and the safety specification \(\mcal{S}\).
% We then use these QCs to set up the safety verification problem as a large \emph{semidefinite program} (SDP) whose satisfiability implies that the network is safe.  

\subsubsection{Abstracting the Input Set via QCs}
To abstract the input set \(\mcal{X}\) into a QC, let \(P(\gamma_{\mrm{in}})\in\mbb{S}^{n_1+1}\) be a symmetric indefinite matrices parametrized by a vector \(\gamma_{\mrm{in}}\in \Gamma_{\mrm{in}}\) where \(\Gamma_{\mrm{in}}\) is the permissible set.
Each such matrix \(P(\gamma_{\mrm{in}})\) now has to satisfy the following QC
\begin{align}\label{eq:input-qc}
    \mbf{z}^\top Z_{\mrm{in}} \mbf{z} \geq 0,
    \quad
    Z_{\mrm{in}}
    \coloneqq
    \begin{bmatrix} E_1 \\ E_a \end{bmatrix}^\top
    P (\gamma_{\mrm{in}})
    \begin{bmatrix} E_1 \\ E_a \end{bmatrix}
    \in \mbb{S}^{N + 1}
\end{align}
where recall that \(E_1\mbf{z} = x_1\) and \(E_a \mbf{z} = 1\).
That is, any \(\mbf{z}\) for which \(x_1 \in \mcal{X}\) must satisfy~\eqref{eq:input-qc} for all \(\gamma_{\mrm{in}} \in \Gamma_{\mrm{in}}\), provided \(P\) is appropriately chosen.
The vector \(\gamma_{\mrm{in}}\) appears linearly in \(P(\gamma_{\mrm{in}})\) and will be a decision variable in DeepSDP.
The dimension of \(\gamma_{\mrm{in}}\) depends on the specific choice of \(P(\gamma_{\mrm{in}})\) (see \cite[Section 3.A]{fazlyab2020safety}).
% : for instance, the polytope \(\mcal{X} \coloneqq\{x_1\in\mbb{R}^n : Hx\le h\}\)  can be abstracted into the QC \eqref{eq:input-qc} with
%     \begin{align*}
%         P(\gamma_{\mrm{in}})\coloneqq\begin{bmatrix}
%             H^\top\Lambda H & -H^\top\Lambda h\\
%             -h^\top\Lambda H & h^\top\Lambda h
%         \end{bmatrix}
%     \end{align*}    
%     where \(\Lambda\) is a symmetric and nonnegative matrix whose entries are those of \(\gamma_{\mrm{in}}\).
%     As the entries of \(\Lambda\) are nonnegative, in this case \(\Gamma_{\mrm{in}}\) is the set of nonnegative vectors. 

\subsubsection{Abstracting the Activation Functions via QCs} To abstract the activation function $\phi:\mbb{R}^{n}\to\mbb{R}^{n}$, let \(Q(\gamma_{\mrm{ac}})\in\mbb{S}^{2n+1}\) be a set of symmetric indefinite matrices linearly parametrized by a vector \(\gamma_{\mrm{ac}}\in \Gamma_{\mrm{ac}}\) where $\Gamma_{\mrm{ac}}$ is the permissible set.
Each such matrix $Q(\gamma_{\mrm{ac}})$ must now satisfy the following QC  
\begin{align}\label{eq:activation}
    \begin{bmatrix}
    u \\
    \phi(u) \\
    1
    \end{bmatrix}^\top
    Q(\gamma_{\mrm{ac}})
    \begin{bmatrix}
    u \\
    \phi(u) \\
    1
    \end{bmatrix}
    \geq 0 \text{ for all } u \in U
\end{align}
where \(U \subseteq \mbb{R}^{n}\) is some input set of \(\phi\).
This is a valid abstraction of \(\phi\) in that any \((u, \phi(u))\) input-output pair will satisfy~\eqref{eq:activation} for all \(\gamma_{\mrm{ac}} \in \Gamma_{\mrm{ac}}\), provided \(Q\) is appropriately chosen.
The dimension of \(\gamma_{\mrm{ac}}\) depends on the particular activation function and on the specific abstraction that is chosen.
As an important example, it is shown in \cite[Proposition 2]{fazlyab2020safety} that many activations \(\varphi\) are \([a, b]\)-\emph{sector-bounded}
\footnote{
    A function \(\varphi\) is \([a,b]\)-sector-bounded if \(a\le \frac{\varphi(u_1)-\varphi(u_2)}{u_1-u_2}\le b\) for all \(u_1,u_2\in\mbb{R}\).
    For instance, the ReLU, sigmoid, and \(\tanh\) activations are all \([0, 1]\)-sector-bounded.
} 
for fixed scalars \(a \leq b\).
In such cases, the stacked activation \(\phi(u) = \text{vcat}(\varphi(u_1), \varphi(u_2), \ldots)\) admits the abstraction
\begin{align}\label{eq:T-sector}
\begin{split}
    Q_{\mrm{sec}}
    &\coloneqq \begin{bmatrix}
        -2ab T & (a + b)T & q_{13} \\
        (a + b)T & -2T & q_{23} \\
        q_{13}^\top  & q_{23} ^\top & q_{33}
    \end{bmatrix} \in \mbb{S}^{2n + 1},
\end{split}
\end{align}
with \(T = \sum_{i=1}^{n} \lambda_{ii} e_i e_i ^\top\), where \(e_i \in \mbb{R}^n\) is the \(i\)th basis vector and \(\lambda_i\) are nonnegative multipliers
\footnote{
The original work of DeepSDP~\cite{fazlyab2020safety} does not enforce a diagonal \(T\), but a counterexample in \cite{pauli2021training} shows that a diagonal \(T\) is the only valid one for sector-bounded abstractions.
% However, a counterexample in \cite{pauli2021training} shows that, under the original sector-bounded abstraction method, the diagonal case is valid while the general one is not.
}
.
In this case, one construction given in~\cite{fazlyab2020safety} takes \(Q(\gamma_{\mrm{ac}}) = Q_{\mrm{sec}}\) with \(q_{13}, q_{23}, q_{33}, \lambda\) linearly parametrized by \(\gamma_{\mrm{ac}} \geq 0\), meaning in this case \(\Gamma_{\mrm{ac}}\) is the set of nonnegative vectors.

Notably, the general activation QC of~\eqref{eq:activation} can be expressed with the entire multi-layered network and not just for a single layer.
By defining \(b \coloneqq \mrm{vcat}(b_1, \ldots, b_{K-1})\),
% {\small
\begin{align*}
% \label{eq:AbB}
    A \coloneqq \begin{bmatrix}
        W_1 & \cdots & 0 & 0 \\
        \vdots & \ddots & \vdots & \vdots\\
        0 & \cdots & W_{K-1} & 0
        \end{bmatrix},
    \
    B \coloneqq \begin{bmatrix}
        0 & I_{n_2} & \cdots & 0 \\
        \vdots & \vdots & \ddots & \vdots \\
        0 & 0 & \cdots & I_{n_K}
        \end{bmatrix},
\end{align*}
% }
we may write \(B \mbf{x} = \phi(A \mbf{x} + b) \in \mbb{R}^{n_2 + \cdots + n_K}\), such that \eqref{eq:activation} then becomes \(\mbf{z}^\top Z_{\mrm{ac}} \mbf{z} \geq 0\) with
\begin{align}\label{eq:deepsdp-Q}
    Z_{\mrm{ac}} \coloneqq
    \begin{bmatrix} \\ \star \\ \\ \end{bmatrix}^\top
    Q(\gamma_{\mrm{ac}})
    \begin{bmatrix} A & b \\ B & 0 \\ 0 & 1 \end{bmatrix}
    \in \mbb{S}^{N + 1},
\end{align}
where \([\star]\) denotes identical terms as on the RHS.
We remark that the DeepSDP framework is not restricted to abstractions of the form \(Q_{\mrm{sec}}\).
However, we will first present our main results under the assumptions of~\eqref{eq:activation} with \(Q(\gamma_{\mrm{ac}}) = Q_{\mrm{sec}}\) in Section \ref{sec:main-results} before showing in Section~\ref{sec:extensions} how other abstractions may be incorporated while preserving sparsity.

\subsubsection{Abstracting the Safety Specification via QCs} To abstract the safety specification \(\mcal{S}\) into a QC, let \(S\in\mbb{S}^{n_1+m+1}\) be a symmetric matrix.
We assume that \(S\) encodes safety by the inequality \(\mbf{z}^\top Z_{\mrm{out}} \mbf{z} \geq 0\) where
\begin{align}\label{eq:safety}
    Z_{\mrm{out}}
    \coloneqq
    \begin{bmatrix} \\ \star \\ \\ \end{bmatrix}^\top
    S
    \begin{bmatrix} I & 0 & 0 \\ 0 & W_K & b_K \\ 0 & 0 & 1 \end{bmatrix}
    \begin{bmatrix} E_1 \\ E_K \\ E_a \end{bmatrix}
    \in \mbb{S}^{N + 1},
\end{align}
which captures a quadratic relation between the input \(x_1\) and output \(y = f(x_1)\).
This expresses a safety condition in that for the input-output pair \((x_1, y)\), the corresponding states \(\mbf{z}\) should satisfy~\eqref{eq:safety}.
For instance, to bound the local \(L_2\) gain via \(\mcal{S} = \{(x_1, y) : \norm{y} \leq \kappa \norm{x_1}\}\), take \(S = \text{blockdiag}(- \kappa^2 I, I, 0)\).

\subsubsection{DeepSDP} Finally, we combine the input, activation, and output QCs and define DeepSDP as the following SDP:
\begin{align}\label{eq:deepsdp}
\begin{split}
    \text{find} &\enskip \gamma \coloneqq (\gamma_{\mrm{in}}, \gamma_{\mrm{ac}}) \in \Gamma_{\mrm{in}}\times \Gamma_{\mrm{ac}} \\
    \text{subject to} &\enskip Z(\gamma) \coloneqq Z_{\mrm{in}} + Z_{\mrm{ac}} + Z_{\mrm{out}} \preceq 0
\end{split}
% \tag{DeepSDP}
\end{align}
We recall the formal safety guarantees of DeepSDP next.

% \vspace{1em}
\begin{lemma}[\cite{fazlyab2020safety}, Theorem 2]
\label{lem:deepsdp-safety-cond}
    Suppose DeepSDP~\eqref{eq:deepsdp} is feasible, then \(f\) satisfies the specification \(\mcal{S}\).
    % \begin{align*}
    %     \begin{bmatrix} x_1 \\ f(x_1) \\ 1 \end{bmatrix}^\top
    %     S
    %     \begin{bmatrix} x_1 \\ f(x_1) \\ 1 \end{bmatrix} \leq 0
    %     \enskip \text{if \(x_1 \in \mcal{X}\)}.
    % \end{align*}
\end{lemma}
This result ensures that the feasibility of DeepSDP implies that the desired safety condition \(S\) holds.
Moreover, one can also derive a \emph{reachability} variant of DeepSDP by parametrizing some components of \(S\), e.g., the \(L_2\) gain \(\kappa\) in the above example to then optimize for the tightest fit.
This notion of reachability is demonstrated in \cite[Section 5]{fazlyab2020safety} and also in Section~\ref{sec:experiments}.
A computational bottleneck of DeepSDP, however, is the large LMI constraint.
In particular, the dimension of \(Z(\gamma)\) grows with the size of the neural network \(f\), and so solving DeepSDP quickly becomes intractable.

%%%%%%%%%%%%%%%%%%%%%%%%%%%%%%%%%%%%%%%%%%%%%%%%%%%%%%%%%%%%%%%%

\subsection{Chordal Sparsity}

The notion of chordal sparsity connects chordal graph theory and sparse matrix decomposition~\cite{griewank1984existence,vandenberghe2015chordal} and can be used to efficiently solve large-scale semidefinite programs.
In the context of this paper, if \(Z(\gamma)\) is chordally sparse, then the LMI constraint \(Z(\gamma) \preceq 0\) within DeepSDP can be split into an equivalent set of smaller matrix constraints.

\subsubsection{Chordal Graphs and Sparse Matrices}

Let \(\mcal{G}(\mcal{V}, \mcal{E})\) be an undirected graph with vertices \(\mcal{V} = [n] \coloneqq \{1, \ldots, n\}\) and edges \(\mcal{E} \subseteq \mcal{V} \times \mcal{V}\).
Since \(\mcal{G}\) is undirected, we assume that \(\mcal{E}\) is symmetric, i.e. \((i,j) \in \mcal{E}\) implies \((j,i) \in \mcal{E}\).
A subset of vertices \(\mcal{C} \subseteq \mcal{V}\) forms a \emph{clique} if \(u,v \in \mcal{C}\) implies that \((u,v) \in \mcal{E}\), and let \(\mcal{C}(i)\) be the \(i\)th vertex of \(\mcal{C}\) under the natural ordering.
If there is no other clique that strictly contains \(\mcal{C}\), then \(\mcal{C}\) is a \emph{maximal clique}.
A \emph{cycle} of length \(k\) is a sequence of vertices \(v_1, \ldots, v_k, v_1 \in \mcal{V}\) with \((v_k, v_1) \in \mcal{E}\) and adjacent connections \((v_i, v_{i+1}) \in \mcal{E}\).
A \emph{chord} is an edge that connects two nonadjacent vertices in a cycle.
We say that a graph is \emph{chordal} if every cycle of length \(\geq 4\) has at least one chord~\cite[Chapter 2]{vandenberghe2015chordal}.

Dually, an edge set \(\mcal{E}\) may also describe the sparsity of a symmetric matrix.
Given \(\mcal{G}(\mcal{V}, \mcal{E})\), the set of symmetric matrices of size \(n\) with sparsity pattern \(\mcal{E}\) is defined as
\begin{align}
\label{eq:SE}
    \mbb{S}^n (\mcal{E})
    \coloneqq
    \braces{ X \in \mbb{S}^n
        : X_{ij} = X_{ji} = 0
        \enskip \text{if \((i,j) \not\in \mcal{E}\)}}.
\end{align}
If \(X \in \mbb{S}^n (\mcal{E})\), we say that \(X\) has sparsity pattern \(\mcal{E}\).
We say that the entry \(X_{ij}\) is \emph{sparse} if \((i,j) \not\in \mcal{E}\), and that it is \emph{dense} otherwise.
Moreover, if \(\mcal{G}(\mcal{V}, \mcal{E})\) is chordal, then we say that \(X\) has \emph{chordal sparsity} or is \emph{chordally sparse}.

\subsubsection{Chordal Decomposition of Sparse Matrices} For a chordally sparse matrix \(X \in \mbb{S}^{n} (\mcal{E})\), one can deduce whether \(X \succeq 0\) by analyzing the maximal cliques of \(\mcal{G}(\mcal{V}, \mcal{E})\).
To do this, given a clique \(\mcal{C} \subseteq \mcal{V}\), define its projection matrix \(E_{\mcal{C}} \in \mbb{R}^{\abs{\mcal{C}} \times n}\) such that \((E_{\mcal{C}})_{ij} = 1\) iff \(\mcal{C}(i) = j\), and \((E_{\mcal{C}})_{ij} = 0\) otherwise.
% This way for \(u \in \mbb{R}^n\) the vector \(E_{\mcal{C}} u \in \mbb{R}^{\abs{\mcal{C}}}\) contains \(u_i\) iff \(i \in \mcal{C}\).
An exact condition for \(X \succeq 0\) is then stated as follows.

% \vspace{1em}
\begin{lemma}[\cite{zheng2019chordal} Theorem 2.10]
\label{lem:chordal-psd}
    Let \(\mcal{G}(\mcal{V}, \mcal{E})\) be a chordal and \(\mcal{C}_1, \ldots, \mcal{C}_p\) its maximal cliques.
    Then \(X \in \mbb{S}^n (\mcal{E})\) and \(X \succeq 0\) iff there are \(X_k \in \mbb{S}^{\abs{\mcal{C}_k}}\) such that each \(X_k \succeq 0\) and
    \(X = \sum_{k = 1}^{p} E_{\mcal{C}_k}^\top X_k E_{\mcal{C}_k}\).
    % \begin{align}
    % \label{eq:X-chordal-decomp}
    %     X = \sum_{k = 1}^{p} E_{\mcal{C}_k}^\top X_k E_{\mcal{C}_k}.
    % \end{align}
\end{lemma}
% We call~\eqref{eq:X-chordal-decomp} a \emph{chordal decomposition} of \(X\).
We call this summation of \(X_k\) a \emph{chordal decomposition} of \(X\).
This is important because we may now solve a large LMI using an equivalent collection of smaller ones.
As a generous lower bound: solving \(X \succeq 0\) is typically \(\Omega(n^3)\) in time~\cite{lofberg2009pre}, meaning that solving a set of smaller \(X_k \succeq 0\) constraints is often desirable.

\section{Network Verification with Chordal-DeepSDP}
\label{sec:main-results}

In this section, we construct a chordally sparse variant of DeepSDP called Chordal-DeepSDP.
We assume here that activation functions are abstracted as sector-bounded functions with \(Q(\gamma_{\mrm{ac}}) = Q_{\mrm{sec}}\), and show in Section~\ref{sec:extensions} how other abstractions can be used while preserving the same sparsity.
The main technical challenge is to precisely characterize the sparsity \(\mcal{E}\) of \(Z(\gamma) \in \mbb{S}^{N+1}\),
and to gain intuition, we show the patterns for different network sizes in Figure~\ref{fig:Z-colors}.

\begin{figure}[ht]
\centering
\begin{minipage}{0.14\textwidth}
    \includegraphics[width=1.0\textwidth]{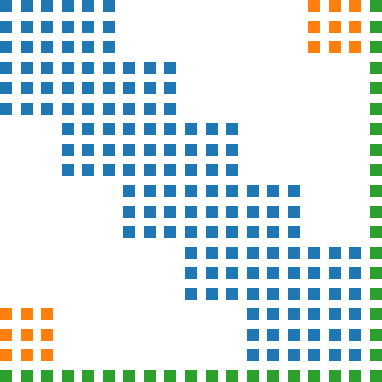}
\end{minipage}%
\quad
\begin{minipage}{0.14\textwidth}
    \includegraphics[width=1.0\textwidth]{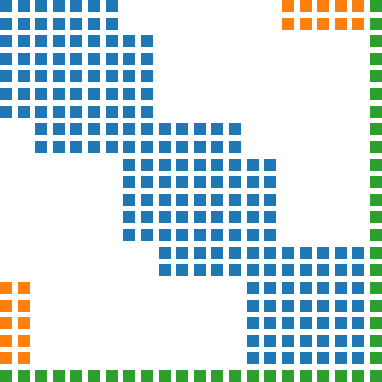}
\end{minipage}%
\quad
\begin{minipage}{0.14\textwidth}
    \includegraphics[width=1.0\textwidth]{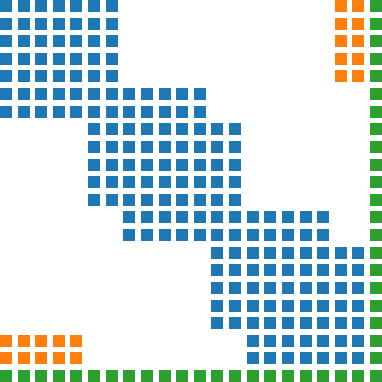}
\end{minipage}

\caption{
The sparsity of \(Z(\gamma) \in \mbb{S}^{N+1}\) for networks of depth \(K=6\) and sizes: left \((3,3,3,3,3,3)\), center \((2,5,2,5,2,5)\), and right \((5,2,5,2,5,2)\).
Each plot corresponds to a sparsity \(\mcal{E}\) where \((i,j) \in \mcal{E}\) iff \((i,j)\) is a colored square.
Each \(\mcal{E}\) is the disjoint union of: \(\mcal{E}_M\) (blue), \(\mcal{E}_a\) (green), and \(\mcal{E}_{1,K}\) (orange).
}
\label{fig:Z-colors}
\end{figure}

It is easy to conjecture that \(\mcal{E}\) is an overlapping block-diagonal pattern with arrow-structure~\cite[Section 8.2]{vandenberghe2015chordal}.
In addition, one can guess that such \(\mcal{E}\) partitions \(Z(\gamma)\) as a block matrix with \(K+1\) blocks, where the first \(K\) blocks each have size \(n_1, \ldots, n_K\) and the last block has size \(1\).
We formalize and prove this intuition: define a disjoint partition of the indices \([N]\) as
\begin{align*}
    [N] = \bigcup_{k=1}^{K} \mcal{V}_k,
    \quad
    \mcal{V}_k = \bigg\{v : 1 + \sum_{j=1}^{k-1} n_j \leq v \leq \sum_{j=1}^{k} n_j\bigg\},
\end{align*}
where \(\mcal{V}_k\) corresponds to the indices (vertices) of the \(k\)th block.
The sparsity \(\mcal{E}\) is then characterized as follows, where we use \(\mcal{A}^2 = \mcal{A} \times \mcal{A}\) to mean the Cartesian product.

\ifautomatica
\vspace{1em}
\fi
\begin{lemma}
\label{lem:Z-sparsity}
    It holds that \(Z(\gamma) \in \mbb{S}^{N+1} (\mcal{E})\), where the sparsity \(\mcal{E} \coloneqq \mcal{E}_M \cup \mcal{E}_a \cup \mcal{E}_{1,K}\) is a disjoint union with
    \begin{align*}
        \mcal{E}_M &\coloneqq \mcal{E}_{M,1} \cup \cdots \cup \mcal{E}_{M,K-1},
        \quad \mcal{E}_{M,k} \coloneqq (\mcal{V}_k \cup \mcal{V}_{k+1})^2, \\
        \mcal{E}_a &\coloneqq \{(i,j) : i = N + 1 \enskip\text{or}\enskip j = N+1\}, \\
        \mcal{E}_{1,K} &\coloneqq \{(i,j) : (i \in \mcal{V}_1, j \in \mcal{V}_{K}) \enskip\text{or}\enskip (i \in \mcal{V}_K, j \in \mcal{V}_1)\}.
    \end{align*}
\end{lemma}

However, a chordal decomposition of \(Z(\gamma)\) given \(\mcal{E}\) is not obvious.
This is because each \(\mcal{G}([N+1], \mcal{E})\) is in fact \emph{not} chordal: the ``wing tips'' at the lower-left and upper-right corners allows for a chord-less cycle of length \(\geq 4\) between vertices \(1\) and \(N+1\), akin to the ``wheel'' shape seen in \cite[Figure 3.1 (right)]{vandenberghe2015chordal}.

Crucially, chordal sparsity is tied to \emph{graphs} rather than to \emph{matrices}.
For instance, if there exists \(\mcal{F} \supseteq \mcal{E}\), then \(Z(\gamma) \in \mbb{S}^{N+1} (\mcal{F})\) as well.
Moreover, if \(\mcal{G}([N+1], \mcal{F})\) is chordal, then one may decompose \(Z(\gamma)\) with respect to the maximal cliques of \(\mcal{G}([N+1], \mcal{F})\) as in Lemma~\ref{lem:chordal-psd}.
Such \(\mcal{F}\) is known as a \emph{chordal extension} of \(\mcal{E}\), and the flexibility of permitting a matrix to have multiple valid sparsity patterns means that we may pick whichever sparsity is convenient (i.e., chordal) to apply chordal decomposition.

\begin{figure}[ht]
% \centering
    % \includegraphics[width=0.29\textwidth]{images/Zbeta0.png}\hspace{1cm}
    % \includegraphics[width=0.29\textwidth]{images/Zbeta2.png}\hspace{1cm}
    % \quad
    % \includegraphics[width=0.29\textwidth]{images/Zbeta4.png}

\centering
\begin{minipage}{0.14\textwidth}
    \includegraphics[width=1.0\textwidth]{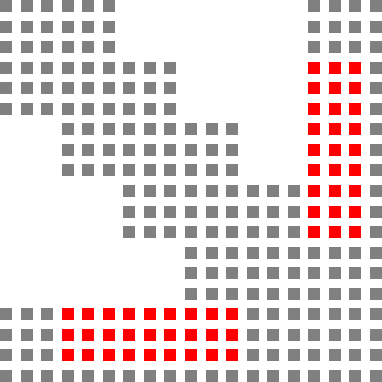}
\end{minipage}%
\quad
\begin{minipage}{0.14\textwidth}
    \includegraphics[width=1.0\textwidth]{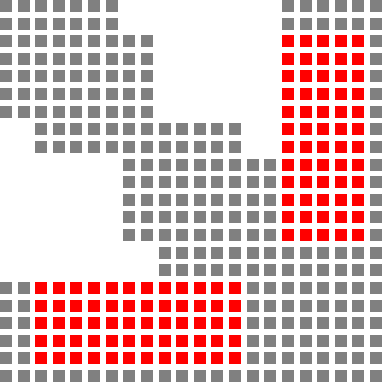}
\end{minipage}%
\quad
\begin{minipage}{0.14\textwidth}
    \includegraphics[width=1.0\textwidth]{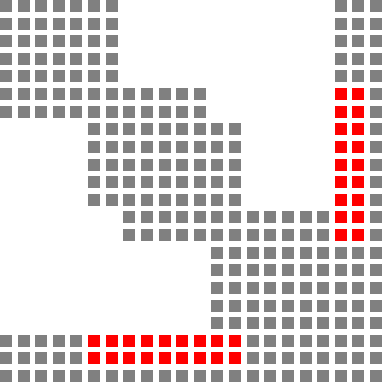}
\end{minipage}
\caption{
A chordal extension of each sparsity \(\mcal{E}\) in Figure~\ref{fig:Z-colors} into a new edge set \(\mcal{F}\), such that each graph \(\mcal{G}([N+1], \mcal{F})\) is now chordal.
Red entries denote the new edges.
}
\label{fig:Z-extensions}
\end{figure}

One such chordal extension for \(\mcal{E}\) is shown in Figure~\ref{fig:Z-colors}, where we fill in certain entries (denoted in red) in order to chordally extend each \(\mcal{E}\) into its corresponding \(\mcal{F}\).
This extension
% \footnote{
% {\color{blue}
%     This is not the unique way to chordally extend \(\mcal{E}\), but this particular choice simplifies our clique analysis.
% }
% }
in Figure~\ref{fig:Z-extensions} can be described by \(\mcal{F} \coloneqq \mcal{E} \cup \mcal{E}_K\) where
\(\mcal{E}_K \coloneqq \{(i,j) : i \in \mcal{V}_K \enskip\text{or}\enskip j \in \mcal{V}_K\}\).
% \begin{align*}
%     \mcal{E}_K &\coloneqq \{(i,j) : i \in \mcal{V}_K \enskip\text{or}\enskip j \in \mcal{V}_K\}
% \end{align*}
The sparsity \(\mcal{F}\) is central to our analysis of \(Z(\gamma)\) and chordal decomposition of DeepSDP.
We next describe the chordality and maximal cliques of \(\mcal{G}([N+1], \mcal{F})\).

\ifautomatica
\vspace{1em}
\fi
\begin{theorem}
\label{thm:Z-cliques}
    The graph \(\mcal{G}([N+1], \mcal{F})\) is chordal with \(p \coloneqq K-2\) maximal cliques, where the \(k\)th maximal clique is given by
    \(\mcal{C}_k \coloneqq \mcal{V}_k \cup \mcal{V}_{k+1} \cup \mcal{V}_{K} \cup \{N+1\}\).
    % \begin{align*}
    %     \mcal{C}_k \coloneqq \mcal{V}_k \cup \mcal{V}_{k+1} \cup \mcal{V}_{K} \cup \{N+1\}.
    % \end{align*}
\end{theorem}

\begin{figure}[ht]
% \centering
    % \includegraphics[width=0.29\textwidth]{images/Zbeta0.png}\hspace{1cm}
    % \includegraphics[width=0.29\textwidth]{images/Zbeta2.png}\hspace{1cm}
    % \quad
    % \includegraphics[width=0.29\textwidth]{images/Zbeta4.png}

\centering
\begin{minipage}{0.14\textwidth}
    \includegraphics[width=1.0\textwidth]{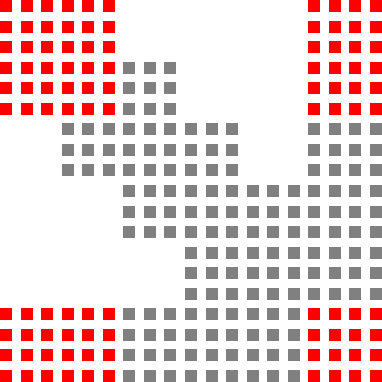}
\end{minipage}%
\quad
\begin{minipage}{0.14\textwidth}
    \includegraphics[width=1.0\textwidth]{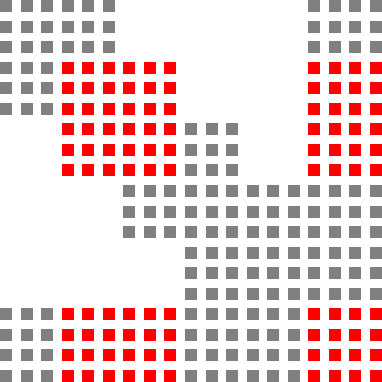}
\end{minipage}%
\quad
\begin{minipage}{0.14\textwidth}
    \includegraphics[width=1.0\textwidth]{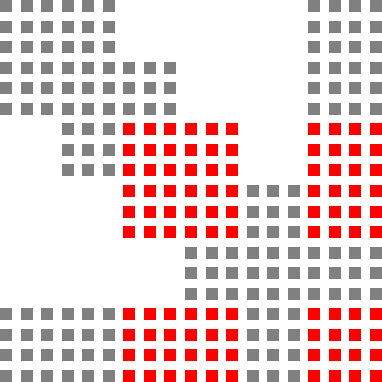}
\end{minipage}
\caption{
The maximal cliques \(\mcal{C}_1, \mcal{C}_2, \mcal{C}_3\) of Figure~\ref{fig:Z-extensions} (left), in red.
The final clique \(\mcal{C}_4\) is the entire lower-right square block.
}
\label{fig:Z-cliques}
\end{figure}

% We prove the above in Appendix~\ref{appendix:main-results}}, which establishes 

This result establishes that \(Z(\gamma) \in \mbb{S}^{N+1} (\mcal{F})\) and that \(\mcal{G}([N+1], \mcal{F})\) is chordal with maximal cliques \(\mcal{C}_1, \ldots, \mcal{C}_p\).
We show an example in Figure~\ref{fig:Z-cliques},
and prove this in the Appendix~\ref{appendix:main-results}.
Importantly, Theorem~\ref{thm:Z-cliques} gives the conditions to apply Lemma~\ref{lem:chordal-psd}, with which we form a chordally sparse variant of DeepSDP that we call Chordal-DeepSDP:
\begin{align}
\label{eq:chordal-deepsdp}
\begin{split}
    \text{find} &\enskip
        \gamma \coloneqq (\gamma_{\mrm{in}}, \gamma_{\mrm{ac}}) \in \Gamma_{\mrm{in}} \times \Gamma_{\mrm{ac}}, \\
    \text{subject to}
        &\enskip Z_1, \ldots, Z_p \preceq 0,
        \enskip
        Z(\gamma) = \sum_{k = 1}^{p} E_{\mcal{C}_k}^\top Z_k E_{\mcal{C}_k}
    % &\enskip
    %     Z(\gamma) = \sum_{k = 1}^{p} E_{\mcal{C}_k}^\top Z_k E_{\mcal{C}_k}
\end{split}
% \tag{Chordal-DeepSDP}
\end{align}
Moreover, by Lemma~\ref{lem:chordal-psd}, we have that DeepSDP and Chordal-DeepSDP are, in fact, \emph{equivalent problems}.

\ifautomatica
\vspace{1em}
\fi
\begin{theorem}
\label{thm:problems-equivalent}
    DeepSDP~\eqref{eq:deepsdp} and Chordal-DeepSDP~\eqref{eq:chordal-deepsdp} are equivalent problems: one is feasible iff the other is.
    That is, \(\gamma\) is a solution to \eqref{eq:deepsdp} iff \(\gamma\) and some \(Z_1, \ldots, Z_p\) is a solution to \eqref{eq:chordal-deepsdp}.
\end{theorem}
\begin{proof}
\ifautomatica
\vspace{-2em}
\fi
Because \(Z(\gamma) \in \mbb{S}^{N+1} (\mcal{F})\) and \(\mcal{G}([N+1], \mcal{F})\) is chordal with maximal cliques \(\mcal{C}_1, \ldots, \mcal{C}_p\), we have by Lemma~\ref{lem:chordal-psd} have \(Z(\gamma) \preceq 0\) iff each \(Z_k \preceq 0\) and
\(Z(\gamma) = \sum_{k = 1}^{p} E_{\mcal{C}_k}^\top Z_k E_{\mcal{C}_k}\).
That is, the constraints of~\eqref{eq:deepsdp} and~\eqref{eq:chordal-deepsdp} are simultaneously feasible.
\end{proof}

\begin{remark}
    If Chordal-DeepSDP~\eqref{eq:chordal-deepsdp} is feasible, then the safety condition in Lemma~\ref{lem:deepsdp-safety-cond} holds.
    For the reachability variant that minimizes a convex objective in \(\gamma\), equivalence means that \eqref{eq:chordal-deepsdp} will compute the same optimal values as DeepSDP~\eqref{eq:deepsdp}.
\end{remark}

In summary, we have decomposed the large semidefinite constraint \(Z(\gamma) \preceq 0\) present in DeepSDP into a large equality constraint and a collection of smaller \(Z_k \preceq 0\) constraints.
The primary benefit of Chordal-DeepSDP is computational since the cost of solving an LMI is usually at least cubic in its size~\cite{lofberg2009pre}.
This means that the cost of solving \(Z(\gamma) \preceq 0\) is \(\Omega(N^3)\), whereas the cost of each \(Z_k \preceq 0\) is \(\Omega(\abs{\mcal{C}_k}^3)\) --- which may be much more efficient.
This is preferable for deeper networks, as the number of cliques grows with the number of layers.
Furthermore, our formulation of Chordal-DeepSDP admits any symmetric \(S\), meaning that we can handle arbitrary quadratically-coupled input-output specifications.
This is a crucial advantage over the more conservative approach of \cite{newton2021exploiting}, which can only handle output constraints.

\section{Bolstering Chordal-DeepSDP with Double Decomposition and Other Abstractions}
\label{sec:extensions}

% In Section \ref{sec:main-results}, we have built-up a basic version of Chordal-DeepSDP.
Having built up a basic version of Chordal-DeepSDP,
we show in Section~\ref{sec:double-decomp} that our formulation of \(\mcal{F}\) from \(\mcal{E}\) admits yet \emph{another} level of chordal decomposition that we found useful for performance --- which we call Chordal-DeepSDP-2.
In addition, our construction until this point has only assumed \(Q(\gamma) = Q_{\mrm{sec}}\) for abstracting activation functions \(\phi\).
This simplifies the presentation of our theoretical results, but in practice, we have found additional abstractions useful to more tightly bound the activation behavior --- in Section~\ref{sec:adj-activations}, we show how to incorporate other activation abstractions that target adjacent-layer connections while preserving the same sparsity \(\mcal{E}\) analyzed in Theorem~\ref{thm:Z-cliques}.
% \red{We refer to our extended manuscript~\cite{xue2022parametric} for detailed technical proofs}.

%%%%%%%%%%%%%%%%%%%%%%%%%%%%%%%%%%%%%%%%%%%%%%%%%%%%%
\subsection{Double Decomposition}
\label{sec:double-decomp}

Recall that we obtained \(\mcal{F}\) from \(\mcal{E}\) by treating certain zero entries of \(Z(\gamma)\) as dense: see Figure~\ref{fig:Z-colors} vs Figure~\ref{fig:Z-extensions}.
If we account for the zero entries, this results in an idea for further sparsifying the smaller \(Z_k \preceq 0\) LMIs.
In particular for \(k = 2, \ldots, p\), suppose we enforce the following structure:
\begin{align}
\label{eq:Zk-sparsed}
    Z_k =
    \begin{bmatrix}
        (Z_k)_{11} & 0 & (Z_k)_{13} \\
        0 & (Z_k)_{22} & (Z_k)_{23} \\
        (Z_k)_{13} ^\top & (Z_k)_{23} ^\top & (Z_k)_{33}
    \end{bmatrix}
    \in \mbb{S}^{\abs{\mcal{C}_k}}
    ,
\end{align}
where \((Z_k)_{11} \in \mbb{S}^{n_k + n_{k+1}}\), \((Z_k)_{22} \in \mbb{S}^{n_K}\), \((Z_k)_{33} \in \mbb{R}\).
By setting \((Z_k)_{12} = 0\), we still maintain the sparsity that \(\sum_{k=1}^{p} E_{\mcal{C}_k}^\top Z_k E_{\mcal{C}_k} \in \mbb{S}^{N+1}(\mcal{E})\).
The advantage now is that each such \(Z_k\) is a block-arrow structure, which is known to be chordal~\cite[Section 8.2]{vandenberghe2015chordal}.
For \(2 \leq k \leq p\), we may then express the chordal decomposition of \(Z_k\) as
\begin{align*}
    Z_{k} = E_{\mcal{D}_{k1}}^\top Y_{k1} E_{\mcal{D}_{k1}}
        + E_{\mcal{D}_2}^\top Y_{k2} E_{\mcal{D}_2}
        \in \mbb{S}^{\abs{\mcal{C}_k}}(\mcal{D}_{k1} \cup \mcal{D}_{k2})
\end{align*}
where the two maximal cliques for this block-arrow shape are \(\mcal{D}_{k1}, \mcal{D}_{k2} \subseteq [\abs{\mcal{C}_k}]\) with
\begin{align*}
    \mcal{D}_{k1} &\coloneqq \{v : 1 \leq v \leq n_k + n_{k+1}\} \cup \{\abs{\mcal{C}_k}\}, \\
    \mcal{D}_{k2} &\coloneqq \{v : n_k + n_{k+1} + 1 \leq v \leq \abs{\mcal{C}_k}\},
\end{align*}
where \(\mcal{D}_{k1}\) covers \((Z_k)_{11}, (Z_k)_{13}, (Z_k)_{33}\) and \(\mcal{D}_{k2}\) covers \((Z_k)_{22}, (Z_k)_{23}, (Z_k)_{33}\).
Applying Lemma~\ref{lem:chordal-psd} again we see that \(Z_k \preceq 0\) iff \(Y_{k1}, Y_{k2} \preceq 0\).
This additional sparsity on \(Z_k\) means we may further chordally decompose Chordal-DeepSDP.
This doubly decomposed problem, which we call Chordal-DeepSDP-2, is formulated as:
\begin{align}
\label{eq:chordal-deepsdp-2}
\begin{split}
    \text{find} &\enskip
        \gamma \coloneqq (\gamma_{\mrm{in}}, \gamma_{\mrm{ac}}) \in \Gamma_{\mrm{in}} \times \Gamma_{\mrm{ac}}, \\
    \text{subject to}
        &\enskip Y_{11}, Y_{12}, \ldots, Y_{p1}, Y_{p2} \preceq 0 \\
        &\enskip
        Z(\gamma) = \sum_{k = 1}^{p}
            E_{\mcal{D}_{k1}}^\top Y_{k1} E_{\mcal{D}_{k1}}
        + E_{\mcal{D}_2}^\top Y_{k2} E_{\mcal{D}_2}
\end{split}
\end{align}
where for notational convenience we let \(Y_{11} = Z_1\) with \(\mcal{D}_{11} = \mcal{C}_1\), and \(Y_{12} = 0\) with \(\mcal{D}_{12} = \emptyset\) in the \(k=1\) case.
Note that under this formulation, we may still use the same \(S\) as before, meaning there is no restriction on the permissible safety specifications.

% \vspace{1em}
\begin{theorem}
\label{thm:chordal-deepsdp-2-soundness}
If \(\gamma\) and \(Y_{11}, \ldots, Y_{p2}\) is a solution for Chordal-DeepSDP-2~\eqref{eq:chordal-deepsdp-2}, then there exists \(\gamma\) and some \(Z_1, \ldots, Z_p \) that is a solution to Chordal-DeepSDP~\eqref{eq:chordal-deepsdp}.
\end{theorem}
\begin{proof}
% \vspace{-1.5em}
Observe by Lemma~\ref{lem:chordal-psd} and \eqref{eq:Zk-sparsed} that each \(Y_{k1}, Y_{k2} \preceq 0\) constraint is equivalent to enforcing \(Z_k \preceq 0\) with \((Z_k)_{12} = 0\) for \(k = 2, \ldots, p\).
Since one can recover \(Z_k\) from each feasible \(Y_{k1}, Y_{k2}\), any solution to~\eqref{eq:chordal-deepsdp-2} therefore implies a solution to~\eqref{eq:chordal-deepsdp}.
\end{proof}

\begin{remark}
\label{rem:double-decomp-equiv}
    We conjecture that Chordal-DeepSDP~\eqref{eq:chordal-deepsdp} and Chordal-DeepSDP-2~\eqref{eq:chordal-deepsdp-2} are equivalent in the sense of Theorem~\ref{thm:problems-equivalent}, as we do not observe any accuracy loss in Chordal-DeepSDP-2 in the experiments.
    However, it is not immediately clear whether solving~\eqref{eq:chordal-deepsdp} necessarily implies that \((Z_k)_{12} = 0\) for \(k = 2, \ldots, p\).
    Instead, this zeroing of \((Z_k)_{12}\) is an additional constraint that is assumed in~\eqref{eq:chordal-deepsdp-2}, meaning that \eqref{eq:chordal-deepsdp-2} satisfiable implies \eqref{eq:chordal-deepsdp} satisfiable, but the converse may not hold.
\end{remark}

%%%%%%%%%%%%%%%%%%%%%%%%%%%%%%%%%%%%%%%%%%%%%%%%%%%%%
\subsection{Adjacent-Layer Abstractions}
\label{sec:adj-activations}

It is often desirable to use multiple activation abstractions to tightly bound the network behavior.
More generally, we say that \(Q_k \in \mbb{S}^{1 + 2n_{k+1}}\) is an abstraction for the adjacent-layer connection \(x_{k+1} = \phi(W_k x_k + b_k)\) if the following inequality holds:
\begin{align}
\label{eq:Qk-term}
    \begin{bmatrix} \\ \star \\ \\ \end{bmatrix}^\top
    Q_k
    \begin{bmatrix} W_k & 0 & b_k \\ 0 & I_{n_{k+1}} & 0 \\ 0 & 0 & 1 \end{bmatrix}
    \begin{bmatrix} E_k \\ E_{k+1} \\ E_a \end{bmatrix}
    \mbf{z}
    \geq 0
\end{align}
for all \(x_k \in \mcal{X}_k\), where \(\mcal{X}_k\) includes the reachable set of \(x_k\) from \(x_1 \in \mcal{X}_1\).
Such \(Q_k\) can encode useful information about the activations at each layer.
For instance, if one knows from \(x_1 \in \mcal{X}\) that \(x_k\) always satisfies \(\underline{x_k} \leq x_k \leq \overline{x_k}\), then \(Q_{k-1}\) can encode the fact that \(x_k\) lies within this box, similar to \cite[Section 3.C.4]{fazlyab2020safety}.
Using interval bounds is popular in practice as they are often cheap to compute, and we use the LiRPA library~\cite{xu2020automatic} for this, though other choices are also valid.

Note that the nonnegativity in \eqref{eq:Qk-term} means we may \emph{add} abstractions to yield new ones.
Thus, the inequality
\begin{align}
\label{eq:Qk-sum}
    \sum_{k = 1}^{K-1}
    \begin{bmatrix} \\ \star \\ \\ \end{bmatrix}^\top
    Q_k
    \begin{bmatrix} W_k & 0 & b_k \\ 0 & I_{n_{k+1}} & 0 \\ 0 & 0 & 1 \end{bmatrix}
    \begin{bmatrix} E_k \\ E_{k+1} \\ E_a \end{bmatrix}
    \mbf{z}
    \geq 0
\end{align}
for all \(x_1 \in \mcal{X}\),
would be an abstraction that accounts for \emph{all} \(K-1\) adjacent-layer connections in the network.
Importantly, the formulation of \eqref{eq:Qk-sum} permits us to use a combination of \emph{arbitrary} abstractions \(Q_k\) for \(\phi\) between adjacent layers, so long as they each satisfy the appropriate QC.
In addition, we may also rewrite~\eqref{eq:Qk-sum} into a form resembling that of~\eqref{eq:deepsdp-Q}, which we show in Lemma~\ref{lem:Qadj}.

% \vspace{1em}
\begin{lemma}
\label{lem:Qadj}
There exists \(Q_{\mrm{adj}}\) linearly parametrized by the \(Q_1, \ldots, Q_{K-1}\) such that the inequality
\begin{align*}
    \mbf{z}^\top
    \begin{bmatrix} A & b \\ B & 0 \\ 0 & 1 \end{bmatrix}^\top
    Q_{\mrm{adj}}
    \begin{bmatrix} A & b \\ B & 0 \\ 0 & 1 \end{bmatrix}
    \mbf{z}
    \geq 0
    \enskip \text{for all} \enskip x_1 \in \mcal{X}
\end{align*}
holds iff the inequality \eqref{eq:Qk-sum} holds.
\end{lemma}
Applying again the nonnegativity condition of activation QCs, we can redefine \(Q(\gamma_{\mrm{ac}}) \coloneqq Q_{\mrm{sec}} + Q_{\mrm{adj}}\).
% Intuitively \(Q_{\mrm{sec}}\) is able to enforce \emph{cross-layer} dependencies, while \(Q_{\mrm{adj}}\) can only enforce adjacent-layer dependencies.
Although different abstractions come with their respective trade-offs, combining them into a single \(Q(\gamma_{\mrm{ac}})\) for DeepSDP means that we can more tightly approximate the activation functions of the entire multi-layered network.
Conveniently, using such a \(Q_{\mrm{adj}}\) does not change the overall sparsity of \(Z(\gamma)\), as we will show next.

% \vspace{1em}
\begin{theorem}
\label{thm:new-Z-sparsity}
    Let \(Q(\gamma_{\mrm{ac}}) \coloneqq Q_{\mrm{sec}} + Q_{\mrm{adj}}\) and similarly define \(Z(\gamma)\) as before in~\eqref{eq:deepsdp}, then \(Z(\gamma) \in \mbb{S}^{N + 1}(\mcal{E})\).
\end{theorem}

% \vspace{-1em}
This means that Chordal-SDP is equivalent to DeepSDP even with \(Q(\gamma_{\mrm{ac}})\).
This sparsity pattern is, therefore, also identical to that of Chordal-DeepSDP-2.

% Consequently, the formulation of Chordal-DeepSDP is still equivalent to DeepSDP under the definition of \(Q(\gamma_{\mrm{ac}})\) in Theorem~\ref{thm:new-Z-sparsity}.
% Moreover, this pattern is identical to that of Chordal-DeepSDP-2.

% Moreover, this sparsity pattern is identical to what is analyzed in the construction of Chordal-DeepSDP-2.

\section{Numerical Experiments}
\label{sec:experiments}

We now evaluate the effectiveness of our approach.
Primarily, we ask (\textbf{Q1}): what is the scalability of Chordal-DeepSDP and Chordal-DeepSDP-2 against DeepSDP?
Secondarily, we investigate (\textbf{Q2}): what is the impact of different QCs on verification accuracy?
For accuracy of (Chordal)-DeepSDP against other methods, see~\cite{fazlyab2020safety}.

% \vspace{-1em}
\paragraph*{Dataset}
% As our focus is on improving the scalability of DeepSDP to deep networks, we 

% As our primary focus is on improving the scalability of DeepSDP using chordal sparsity, we focus on deep networks.

To study the scalability of (Chordal)-DeepSDP on deep networks, we consider the reachable set of discrete-time dynamical systems \(x(t+1) = f(x(t))\) described by a neural network \(f : \mbb{R}^n \to \mbb{R}^n\) and an initial state \(x(0) \in \mcal{X}\).
By self-composing \(f^{(t)} = f \circ \cdots \circ f\) we obtain arbitrary \(t\)-length trajectories.
As  a non-trivial system we discretize the continuous-time friction-less cart-pole dynamics~\cite{florian2007correct} using parameters from~\cite{chen2022one}:
cart mass \(0.25\) kg, pole mass \(0.1\) kg, length \(0.4\) m, time step \(dt = 0.05\) s, and initial states \(\mcal{X} \subseteq \mbb{R}^4\) where each \(x \in \mcal{X}\) is constrained as
\(x_1 \in [2.0, 2.2] \, \mrm{m}\),
\(x_2 \in [1.0, 1.2] \, \mrm{m/s}\),
\(x_3 \in [-0.174, -0.104] \, \mrm{rad}\),
and \(x_4 \in [-1.0, -0.8] \, \mrm{rad/s}\).

% For a non-trivial system, we consider the continuous-time cart-pole dynamics without friction from~\cite{florian2007correct}, which is a system with 4-dimensional states and 1-dimensional control.
% We use the configuration from~\cite{chen2022one}:
% \begin{align}
% \label{eq:cartpole-X0}
% \begin{split}
%     \mcal{X} = \{x \in \mbb{R}^4 : \enskip
%         &x_1 \in [2.0, 2.2] \, \mrm{m}, \\
%         & x_2 \in [1.0, 1.2] \, \mrm{m/s}, \\
%         &x_3 \in [-0.174, -0.104] \, \mrm{rad}, \\
%         & x_4 \in [-1.0, -0.8] \, \mrm{rad/s}
%     \}.
% \end{split}
% \end{align}

To simulate the state trajectory, we trained two neural networks with ReLU activations of layer sizes \(\{4,40,40,40,40,4\}\) (\(\msf{Cart40}\)) and \(\{4,10,10,10,10,4\}\) (\(\msf{Cart10}\)).
For instance \(\msf{Cart40}^{(3)} : \mbb{R}^4 \to \mbb{R}^4\) is a 3-step simulation of the cart-pole dynamics with \(\msf{Cart40}\), which will have \(4 \times 3 = 12\) hidden layers of size \(40\) each.
We trained two networks because DeepSDP will sometimes, surprisingly, run out of memory for even \(\msf{Cart40}^{(3)}\) (LMI size \(= 4 + 12 \times 40 = 484\)).
% \(\msf{Cart10}\) is thus used to get more comprehensive scalability experiments.

% To yield a closed-loop system, we use the zero-controller \(u \equiv 0\), and train two feedforward neural networks with ReLU activations of layer sizes \(\{4,40,40,40,40,4\}\) (\(\msf{Cart40}\)) and \(\{4,10,10,10,10,4\}\) (\(\msf{Cart10}\)).

% By composing these networks with themselves, we yield discrete-time simulations of the cart-pole dynamics, e.g.
% \(\msf{Cart40}^{(3)} \coloneqq \msf{Cart40} \circ \msf{Cart40} \circ \msf{Cart40}\)
% \begin{align*}
%     \msf{Cart40}^{(3)} = \msf{Cart40} \circ \msf{Cart40} \circ \msf{Cart40} : \mbb{R}^{4} \to \mbb{R}^{4}
% \end{align*}

% For instance \(\msf{Cart40}^{(3)} : \mbb{R}^4 \to \mbb{R}^4\) is a 3-step simulation of the cart-pole dynamics with \(\msf{Cart40}\), which will have \(4 \times 3 = 12\) hidden layers of size \(40\) each.
% We trained two neural networks because DeepSDP will often (surprisingly) run out-of-memory for even \(\msf{Cart40}^{(3)}\) (LMI size \(= 4 + 12 \times 40 = 484\)), so the thinner \(\msf{Cart10}\) is used to get more comprehensive scalability experiments.
% }

% \vspace{-1em}
\paragraph*{QC Configuration}
All networks in our experiments use ReLU activations, and we use three types of activation QCs, all of which will have sparsity \(\mcal{E}\):
\begin{itemize}
    \item Sector-bounded QCs for ReLUs as in \cite[Lemma 4]{fazlyab2020safety}.
    These can be expressed via \(Q_{\mrm{sec}}\), and we remark that similar ones for tanh can be found in \cite[Section 3]{fazlyab2020safety}.

    \item Bounded-nonlinearity QCs as in \cite[Section 3.C.4]{fazlyab2020safety}, where we use interval propagation bounds obtained from LiRPA~\cite{xu2020automatic}.
    These can be expressed via \(Q_{\mrm{adj}}\).

    \item Final-output QCs that bound \(\underline{y} \leq y \leq \overline{y}\), which take similar form to the bounded-nonlinearity QCs.
    We again use information from LiRPA.
\end{itemize}

\begin{figure}[h]
\centering
\begin{subfigure}{0.45\textwidth}
\includegraphics[width=1.0\textwidth]{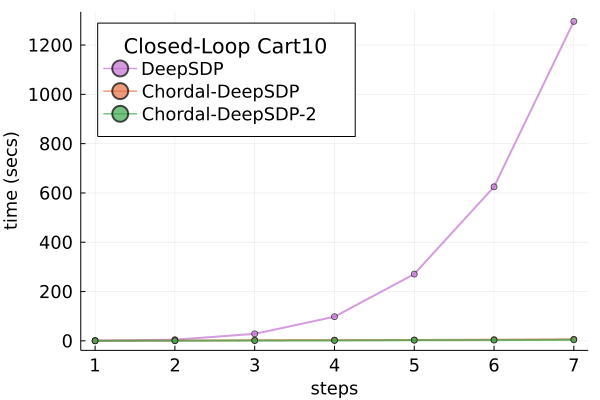}
\end{subfigure}

\begin{subfigure}{0.45\textwidth}
\includegraphics[width=1.0\textwidth]{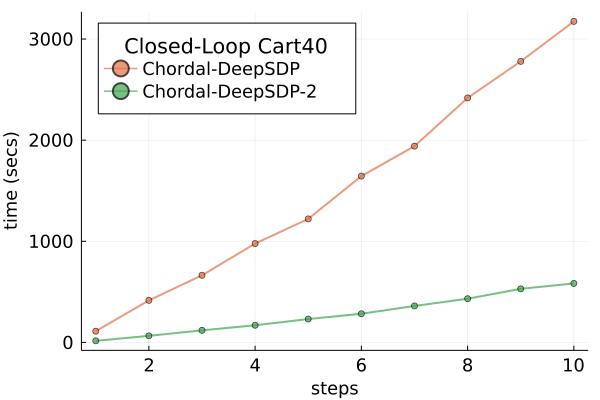}
\end{subfigure}

\caption{(Top)
Runtimes of \(\msf{Cart10}^{(t)}\) for \(t = 1, \ldots, 7\) steps.
(Bottom) Runtimes of \(\msf{Cart40}^{(t)}\) for \(t = 1, \ldots, 10\) steps.
}
% \caption{Runtimes of \(\msf{Cart10}^{(t)}\) for \(t = 1, \ldots, 7\) steps.}
\label{fig:scale-deep-chordal-chordal2}
\end{figure}

% \vspace{-1em}
\paragraph*{System}
All experiments were run on a compute server with Intel Xeon Gold 6148 CPU @ 2.40 GHz with 80 cores and 692 GB of RAM.
We used Julia 1.7.2~\cite{bezanson2017julia} for our implementation with MOSEK 9.3~\cite{andersen2000mosek} as our convex solver.
Relevant Julia packages include \texttt{JuMP.jl} v1.0.0~\cite{dunning2017jump} and \texttt{Dualization.jl} v0.5.3.
% \footnote{\texttt{https://github.com/jump-dev/Dualization.jl}}

\subsection{(\textbf{Q1}): Scalability of Chordal Decomposition}

For scalability, we test how fast each method can find the optimal offset of the hyperplane with normal of the basis vector \(e_1 \in \mbb{R}^{4}\) and initial conditions above.
The corresponding setup for the safety specification is described in~\cite[Section 5.A]{fazlyab2020safety} and is equivalent to computing the top-edge of the bounding boxes in Figure~\ref{fig:reach-trajs}.

Since DeepSDP runs out-of-memory for even \(\msf{Cart40}^{(3)}\), we first use \(\msf{Cart10}\) for a comprehensive comparison of DeepSDP, Chordal-DeepSDP, and Chordal-DeepSDP-2, with results shown in Figure~\ref{fig:scale-deep-chordal-chordal2}.
In the top plot, it can be seen that chordal decomposition via Chordal-DeepSDP and Chordal-DeepSDP-2 significantly outperforms DeepSDP.
The bottom plot shows that the double-decomposition of Chordal-DeepSDP-2 is helpful and further improves scalability over Chordal-DeepSDP.
For DeepSDP we show only up to \(t = 7\) steps, because \(\msf{Cart10}^{(8)}\) often runs out-of-memory.

Interestingly, we found that preprocessing DeepSDP with the \texttt{Dualization.jl} package allows us to achieve a much faster runtime, which we show for \(\msf{Cart40}\) at higher steps in Figure~\ref{fig:scale-deepdual-chordal2}.
This is a surprising phenomenon we discuss in more depth in Section~\ref{sec:discussion}.

\begin{figure}[h]
\centering

\begin{subfigure}{0.45\textwidth}
\includegraphics[width=1.0\textwidth]{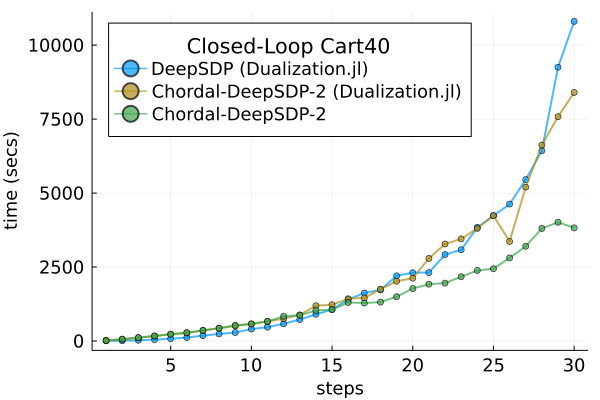}
\end{subfigure}

\caption{
Runtimes of \(\msf{Cart40}^{(t)}\) for \(t = 1, \ldots, 30\) steps.
We do not plot Chordal-DeepSDP here, as it is consistently slower than DeepSDP (\texttt{Dualization.jl}) by more than \(\times 3\), and instead focus on comparison with Chordal-DeepSDP-2.
}
\label{fig:scale-deepdual-chordal2}
\end{figure}

\subsection{(\textbf{Q2}): Activation QCs on Accuracy}

We next investigate the impact of different activation QCs on accuracy to demonstrate that our adjacent-layer abstractions effectively improve accuracy.
First, in Figure~\ref{fig:reach-trajs}, we show a plot of the bounding boxes for \(\msf{Cart40}\) with initial conditions from earlier, and using all three QCs (bounded-nonlinearity, sector-bounded, and final-bounded).
This shows that DeepSDP and its chordally sparse variants can yield reasonably good bounds for state reachability.

\begin{figure}[ht]
\centering
\begin{subfigure}{0.45\textwidth}
\includegraphics[width=1.0\textwidth]{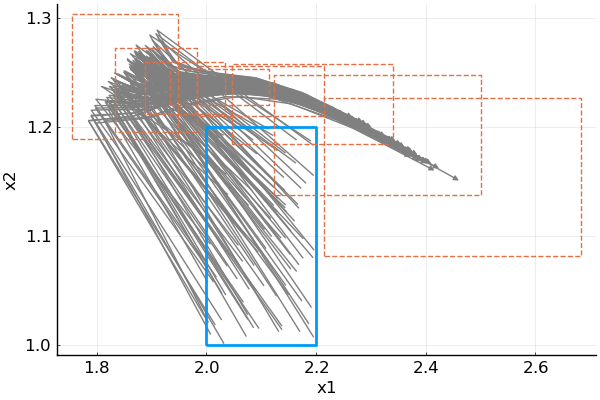}
\end{subfigure}

\begin{subfigure}{0.45\textwidth}
\includegraphics[width=1.0\textwidth]{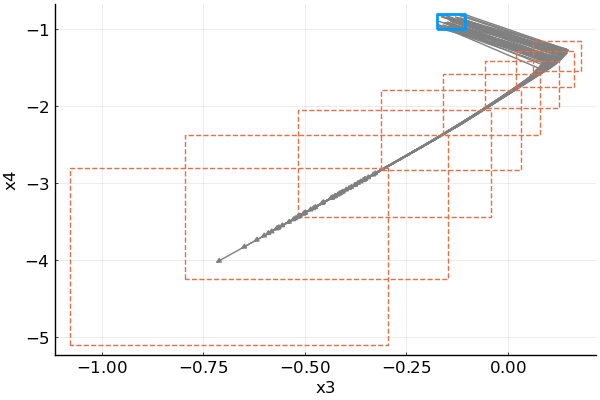}
\end{subfigure}

\caption{Bounds on \(\msf{Cart40}^{(k)}\) for \(k = 1, \ldots, 8\) using sector-bounded, bounded-nonlinearities, and final-bounded QCs.
}
\label{fig:reach-trajs}
\end{figure}

We also in show in Figure~\ref{fig:l2gain-qc-activs} the local \(L_2\) gain of \(\msf{Cart40}^{(t)}\) for \(t = 1, \ldots, 5\) with initial conditions \(\mcal{X} = \{x : 0.5 \leq x \leq 1.5\}\), for different combinations of activation QCs.
We use the DeepSDP with only bounded-nonlinearity QC as a baseline and find that a combination of activations yields tighter bounds for the \(L_2\) gain.

\begin{figure}[ht]
\centering
\begin{subfigure}{0.45\textwidth}
\includegraphics[width=1.0\textwidth]{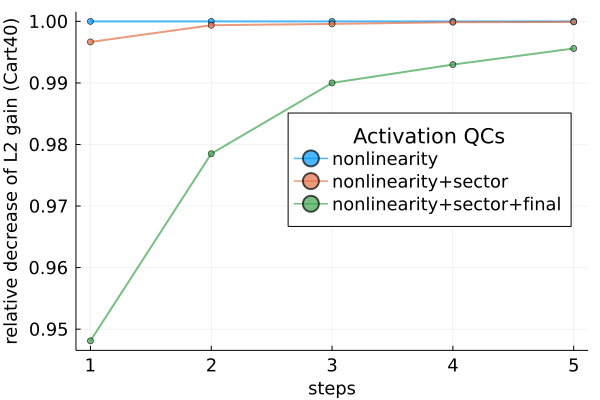}
\end{subfigure}
\caption{The \(L_2\) gain of \(\msf{Cart40}^{(t)}\) for \(t = 1, \ldots, 5\).
We use the bounded-nonlinearity QC as a baseline (1.0), and plot the relative decrease (improvement) when we add the sector-bounded QC and the final-bounded QC.
To see that the bounds are not vacuous, the baseline values of the \(L_2\) gain are: \(2.616, 2.489, 2.552, 2.666, 2.795\).
We remark that in this example, sector-bounded QCs alone are too weak, and will make DeepSDP and its chordal decompositions infeasible.
}
\label{fig:l2gain-qc-activs}
\end{figure}

\section{Discussion}
\label{sec:discussion}

Our experiments show that Chordal-DeepSDP-2 outperforms both DeepSDP and Chordal-DeepSDP on deep networks.
The double decomposition of Chordal-DeepSDP-2 is especially helpful for the experiment of Figure~\ref{fig:scale-deepdual-chordal2}, for which Chordal-DeepSDP is much slower than DeepSDP (\texttt{Dualization.jl}).
Moreover, Chordal-DeepSDP-2 exhibits the best scaling in all cases.

Remarkably, we found that preprocessing DeepSDP with \texttt{Dualization.jl} yields a significant performance gain (often over \(\times 10\) faster).
This is surprising because MOSEK uses a primal-dual interior point method for solving SDPs, so dualizing a problem should not yield any performance gains.
In particular, we suspect that the \texttt{Dualization.jl} transformations on DeepSDP may be highly amenable to MOSEK's internal heuristics.
In fact, DeepSDP (with \texttt{Dualization.jl}) will outperform Chordal-DeepSDP on nearly all the experiments of Figure~\ref{fig:scale-deepdual-chordal2}, as well as Chordal-DeepSDP-2 on sufficiently shallow networks.
% However, we believe that the benefits of \texttt{Dualization.jl}

% Unfortunately, the same \texttt{Dualization.jl} preprocessing for Chordal-DeepSDP and Chordal-DeepSDP-2 only hinders performance, see Figure~\ref{fig:scale-deepdual-chordal2}.

A limitation of our approach is that it does not decompose wide layers:
even with Chordal-DeepSDP-2, each \(Y_{k1} \preceq 0\) constraint has size \(n_k + n_{k+1} + 1\).
Therefore, wider layers induce larger \(Z_k\), which is not preferable.
For our experiments, however, we unroll closed-loop dynamics over long time horizons to yield deep networks, meaning that this is not an issue as long as the initial network is not too wide.

% Additionally, we conjectured in Remark~\ref{rem:double-decomp-equiv} that the double decomposition scheme preserves equivalence between Chordal-DeepSDP-2 and Chordal-DeepSDP --- so by extension DeepSDP.
% This is in part because we have not found an instance where the two differ.

There are several interesting future directions.
For instance, the accuracy of DeepSDP-style verification~\cite{fazlyab2020safety,fazlyab2019efficient} depends on the choice of activation abstraction, and new abstractions beyond sector-bounded inequalities would help improve accuracy.
In addition, a promising recent development is the use of polynomial optimization, specifically sum-of-squares~\cite{parrilo2000structured} techniques, for neural network verification~\cite{brown2022unified,chen2020semialgebraic,newton2021neural,newton2022sparse}.
Polynomial constraints are more expressive than what is easily encodable in DeepSDP and can improve verification accuracy in principle.
However, polynomial optimization techniques at present must rely on a large SDP as their primary computational mechanism, so exploiting problem-specific feaures~\cite{lofberg2009pre} and sparsity~\cite{newton2022sparse} is thus crucial for performance.

\section{Conclusion}

We present Chordal-DeepSDP, a chordally sparse decomposition of DeepSDP that achieves computational scalability in the network depth at zero loss of accuracy and without sacrificing safety constraint expressiveness.
Additionally, we present analysis techniques that lead to yet another level of decomposition that we call Chordal-DeepSDP-2.
We further present a framework for incorporating additional activation constraints via adjacent-layer abstractions, allowing us to tighten the verification accuracy.
Our experiments demonstrate the scalability advantages of chordal decompositions for DeepSDP.

% Our numerical evaluations demonstrate the effectiveness of our approach and illustrate the computational advantages of Chordal-DeepSDP and Chordal-DeepSDP-2 over DeepSDP.

\bibliographystyle{IEEEtran}
\bibliography{sources}

% \newpage

\appendix

\section{Proofs of Section~\ref{sec:main-results}}
\label{appendix:main-results}

\subsection{Results for Lemma~\ref{lem:Z-sparsity}}

\begin{proof}[Proof of Lemma~\ref{lem:Z-sparsity}]
First, it is easy to see that \(\mcal{E}\) is a disjoint union by inspection.
Next, we have that
\begin{align*}
    Z_{\mrm{in}} &\in \mbb{S}^{N+1} (\mcal{E}_M \cup \mcal{E}_a),
        \tag{Lemma~\ref{lem:Zin-sparsity}} \\
    Z_{\mrm{ac}} &\in \mbb{S}^{N+1} (\mcal{E}_M \cup \mcal{E}_a),
        \tag{Lemma~\ref{lem:Zac-sparsity}} \\
    Z_{\mrm{out}} &\in \mbb{S}^{N+1} (\mcal{E}_M \cup \mcal{E}_a \cup \mcal{E}_{1,K}),
    \tag{Lemma~\ref{lem:Zout-sparsity}}
\end{align*}
and so the sum \(Z_{\mrm{in}} + Z_{\mrm{ac}} + Z_{\mrm{out}} = Z(\gamma) \in \mbb{S}^{N+1} (\mcal{E})\).
\end{proof}

\begin{remark}
\label{rem:EtXE-expand}
Suppose a graph has \(n = 4\) vertices and clique \(\mcal{C} = \{1, 2, 4\}\), then the projection matrix \(E_{\mcal{C}}\) has the following property for any \(X \in \mbb{S}^3\)
that is convenient for checking dense entries:
\begin{align*}
    E_{\mcal{C}}^\top X E_{\mcal{C}}
    = \begin{bmatrix}
        X_{11} & X_{12} & 0 & X_{14} \\
        X_{21} & X_{22} & 0 & X_{24} \\
        0 & 0 & 0 & 0 \\
        X_{41} & X_{42} & 0 & X_{44}
    \end{bmatrix}.
\end{align*}
% This is convenient for quickly determining dense entries.
\end{remark}

\begin{lemma}
    \label{lem:Zin-sparsity}
    It holds that \(Z_{\mrm{in}} \in \mbb{S}^{N+1}(\mcal{E}_M \cup \mcal{E}_a)\).
\end{lemma}
\begin{proof}
    Recall \eqref{eq:input-qc} and observe that \((Z_{\mrm{in}})_{ij}\) is dense iff exactly one of the following disjoint conditions holds:
    \begin{enumerate}
        \item[(1)] \(i \in \mcal{V}_1\) and \(j \in \mcal{V}_1\)
        \item[(2)] \(i \in \mcal{V}_1\) and \(j = N+1\) (or the symmetric case)
        \item[(3)] \(i = N+1\) and \(j = N+1\)
    \end{enumerate}
    If (1) holds, then \((i,j) \in \mcal{E}_{M,1}\); if any of (2) or (3) holds, then \((i,j) \in \mcal{E}_a\).
    Thus we have that \((i,j)\) dense implies \((i,j) \in \mcal{E}_M \cup \mcal{E}_a\), and so \(Z_{\mrm{in}} \in \mbb{S}^{N+1} (\mcal{E}_M \cup \mcal{E}_a)\).
\end{proof}

% \red{
% \begin{proof}
%    Recall \eqref{eq:input-qc} and observe that \((Z_{\mrm{in}})_{ij}\) is dense iff
%     \begin{align*}
%         i, j \in \underbrace{\{v : 1 \leq v \leq n_1\}}_{\mcal{B}_1}
%         \cup \underbrace{\{v : v = N + 1\}}_{\mcal{B}_a}
%     \end{align*}
%     where \(\mcal{B}_1, \mcal{B}_a \subseteq [N+1]\) are disjoint.
%     Equivalently, \((Z_{\mrm{in}})_{ij}\) is dense iff any of the following conditions hold:
%     \begin{enumerate}
%         \item[(1)] \(i \in \mcal{B}_1\) and \(j \in \mcal{B}_1\)
%         \item[(2)] \(i \in \mcal{B}_1\) and \(j \in \mcal{B}_a\) (or the symmetric case)
%         \item[(3)] \(i \in \mcal{B}_a\) and \(j \in \mcal{B}_a\)
%     \end{enumerate}
%     If (1) holds, then \((i,j) \in \mcal{E}_{M, 1}\); if any of (2) or (3) holds, then \((i,j) \in \mcal{E}_a\).
%     Thus \((i,j)\) dense implies \((i,j) \in \mcal{E}_M \cup \mcal{E}_a\), and so \(Z_{\mrm{in}} \in \mbb{S}^{N+1} (\mcal{E}_M \cup \mcal{E}_a)\).
% \end{proof}
% }

\begin{lemma}
\label{lem:Zac-sparsity}
    It holds that \(Z_{\mrm{ac}} \in \mbb{S}^{N+1} (\mcal{E}_M \cup \mcal{E}_a)\).
\end{lemma}
\begin{proof}
Recall \eqref{eq:deepsdp-Q} and express \(Z_{\mrm{ac}} \in \mbb{S}^{N+1}\) as:
\begin{align*}
    Z_{\mrm{ac}}
    &=
    \left[\begin{array}{c|c}
        A & b \\ B & 0 \\ \hline 0 & 1
    \end{array}\right]^\top
    \left[\begin{array}{cc|c}
        a_{11} T & a_{12} T & q_{13} \\
        a_{12} T & a_{22} T & q_{23} \\
        \hline
        \star & \star & q_{33}
    \end{array}\right]
    \left[\begin{array}{c|c}
        A & b \\ B & 0 \\ \hline 0 & 1
    \end{array}\right] \\
    &=
    \left[\begin{array}{c|c}
        \begin{bmatrix} A \\ B \end{bmatrix}^\top
        \begin{bmatrix} a_{11} T & a_{12} T \\ a_{12} T & a_{22} T \end{bmatrix}
        \begin{bmatrix} A \\ B \end{bmatrix}
        & \star \\
        \hline
        \star & \star
    \end{array}\right]
    =
    \left[\begin{array}{c|c}
        M & p_{12} \\ \hline \star & p_{22}
    \end{array}\right]
    % &=
    % \left[\begin{array}{c|c}
    %     M & p_{12} \\ \hline \star & p_{22}
    % \end{array}\right]
    % \in \mbb{S}^{N+1}
\end{align*}
where \(p_{12} \in \mbb{R}^{N}\) and \(p_{22} \in \mbb{R}\) follow from straightforward calculations.
From Lemma~\ref{lem:M-sparsity} we have \(M \in \mbb{S}^{N}(\mcal{E}_M)\), and so it follows that \(Z_{\mrm{ac}} \in \mbb{S}^{N+1}(\mcal{E}_M \cup \mcal{E}_a)\).
\end{proof}

\begin{lemma}
\label{lem:Zout-sparsity}
It holds that \(Z_{\mrm{out}} \in \mbb{S}^{N+1} (\mcal{E}_M \cup \mcal{E}_a \cup \mcal{E}_{1,K})\).
\end{lemma}
\begin{proof}
Recall \eqref{eq:safety} and observe that \((Z_{\mrm{out}})_{ij}\) is dense iff exactly one of the following disjoint conditions holds:
\begin{enumerate}
    \item[(1)] \(i \in \mcal{V}_1\) and \(j \in \mcal{V}_1\)
    \item[(2)] \(i \in \mcal{V}_1\) and \(j \in \mcal{V}_K\) (or the symmetric case)
    \item[(3)] \(i \in \mcal{V}_1\) and \(j = N+1\) (or the symmetric case)
    \item[(4)] \(i \in \mcal{V}_K\) and \(j \in \mcal{V}_K\)
    \item[(5)] \(i \in \mcal{V}_K\) and \(j = N+1\) (or the symmetric case)
    \item[(6)] \(i = N+1\) and \(j = N+1\)
\end{enumerate}
If (1) holds, then \((i,j) \in \mcal{E}_{M,1}\);
if (4) holds, then \((i,j) \in \mcal{E}_{M, K-1}\);
if (2) holds, then \((i,j) \in \mcal{E}_{1,K}\);
if any of (3), (5), or (6) holds, then \((i,j) \in \mcal{E}_a\).
Thus if \((i,j)\) is dense, then \((i,j) \in \mcal{E}_M \cup \mcal{E}_a \cup \mcal{E}_{1,K}\).
Consequently we have that \(Z_{\mrm{out}} \in \mbb{S}^{N+1} (\mcal{E}_M \cup \mcal{E}_a \cup \mcal{E}_{1,K})\).
\end{proof}

\begin{lemma}
\label{lem:M-sparsity}
    For all \(a_{11}, a_{12}, a_{22} \in \mbb{R}\) it holds that
    \begin{align*}
        M \coloneqq
        \begin{bmatrix} A \\ B \end{bmatrix}^\top
        \begin{bmatrix} a_{11} T & a_{12} T \\ a_{12} T & a_{22} T \end{bmatrix}
        \begin{bmatrix} A \\ B \end{bmatrix}
        \in \mbb{S}^N (\mcal{E}_M).
    \end{align*}
\end{lemma}
\begin{proof}
We first rewrite the expression for \(M\) as
\begin{align*}
% \label{eq:M-sum}
    a_{11} A^\top T A + a_{12} A^\top T B + a_{12} B^\top T A + a_{22} B^\top T B,
\end{align*}
where we will first present a decomposition of \(A, B, T\), and then analyze each term separately.
Now observe that we may express \(A\) and \(B\) as
\begin{align*}
    A = \sum_{k = 1}^{K-1} F_k ^\top W_k G_k,
    \quad B = \sum_{k = 1}^{K-1} F_k ^\top G_{k+1},
\end{align*}
where \(F_k \in \mbb{R}^{n_{k+1} \times (n_2 + \cdots + n_K)}\) is the \(k\)th block-index selector over \(\{n_2, \ldots, n_K\}\),
and \(G_k \in \mbb{R}^{n_k \times (n_1 + \cdots + n_K)}\) is the \(k\)th block index selector over \(\{n_1, \ldots, n_K\}\).
Similarly, since \(T\) is diagonal, we may write:
% \(T = \sum_{k = 1}^{K-1} F_k ^\top T_k F_k\),
\begin{align*}
    T = \sum_{k = 1}^{K-1} F_k ^\top T_k F_k,
\end{align*}
where each \(T_k\) is itself diagonal.
Moreover, the \(F_k\) matrices are orthogonal in the sense that:
\begin{align*}
    F_k F_k ^\top = I, &\enskip\text{and}\enskip F_j F_k ^\top = 0 \enskip \text{when \(j \neq k\)},
\end{align*}
and similar conditions hold for \(G_k\).
By orthogonality of the \(F_k\), we have the following for each term of \(M\):
\begin{align}
    % TA &= \sum_{k=1}^{K-1} F_k ^\top T_k W_k G_k, \label{eq:TA} \\
    A^\top T A &= \sum_{k=1}^{K-1} G_k ^\top W_k ^\top T_k W_k G_k, \label{eq:ATA} \\
    B^\top T A &= \sum_{k=1}^{K-1} G_{k+1} T_k W_k G_k, \label{eq:BTA} \\
    B^\top T B &= \sum_{k=1}^{K-1} G_{k+1} T_k G_{k+1}, \label{eq:BTB}
\end{align}
where \(A^\top T B\) follows from~\eqref{eq:BTA} by symmetry.
Grouping each summation term of \eqref{eq:ATA}, \eqref{eq:BTA}, \eqref{eq:BTB}, we have that \(M = \sum_{k = 1}^{K-1} M_k\) where each
\begin{align*}
    % M_k &= a_{11} G_k ^\top W_k ^\top T_k W_k G_k
    %     + a_{22} G_{k+1} T_k G_{k+1} \\
    %   &\qquad + a_{12} \parens{G_{k+1} ^\top T_k W_k G_k + G_k ^\top W_k ^\top T_k G_{k+1}}
    %   % &\qquad + a_{22} G_{k+1} T_k G_{k+1}
    M_k &= a_{11} G_k ^\top W_k ^\top T_k W_k G_k \\
      &\quad + a_{12} \parens{G_{k+1} ^\top T_k W_k G_k + G_k ^\top W_k ^\top T_k G_{k+1}} \\
      &\quad + a_{22} G_{k+1} T_k G_{k+1}
\end{align*}
We now identify each \(G_k \in \mbb{R}^{n_k \times (n_1 + \cdots + n_K)}\) with the projections \(E_{\mcal{C}}\) described in Remark~\ref{rem:EtXE-expand}, and with this we see that each \(G_k\) is a projection for \(\mcal{V}_k\).
Since \(N = n_1 + \cdots + n_K\), we thus have \(M_k \in \mbb{S}^{N}(\mcal{E}_{M,k})\), and so
\begin{align*}
    M = \sum_{k=1}^{K-1} M_k
    \in \mbb{S}^{N}\parens{\bigcup_{k=1}^{K-1} \mcal{E}_{M,k}}
    = \mbb{S}^{N} (\mcal{E}_M),
\end{align*}
with the definitions of \(\mcal{E}_M, \mcal{E}_{M,k}\) given in Lemma~\ref{lem:Z-sparsity}.
\end{proof}

\subsection{Results for Theorem~\ref{thm:Z-cliques}}

\begin{figure}[ht]
\centering
\begin{subfigure}{0.14\textwidth}
\begin{center}
\includegraphics[width=1.0\textwidth]{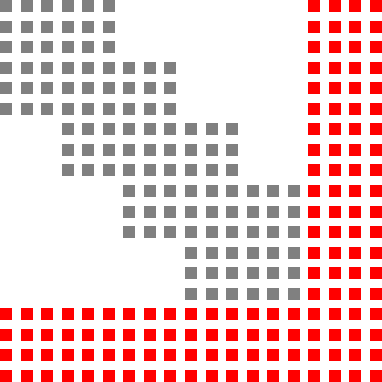}
\end{center}
\end{subfigure}
\caption{The disjoint partitions of \(\mcal{F}_A \coloneqq \bigcup_{k=1}^{K-2} \mcal{E}_{M,k}\) (gray) and \(\mcal{F}_B \coloneqq \mcal{E}_K \cup \mcal{E}_a\) (red).
}

%     \mcal{F}_A \coloneqq \bigcup_{k=1}^{K-2} \mcal{E}_{M,k},
%     \quad \mcal{F}_B \coloneqq \mcal{E}_K \cup \mcal{E}_a,
\label{fig:FA-FB}
\end{figure}

\begin{proof}[Proof of Theorem~\ref{thm:Z-cliques}]
% We may partition \(\mcal{F} = \mcal{F}_A \cup \mcal{F}_B\) with \(\mcal{F}_A \cap \mcal{F}_B = \emptyset\), such that
% \begin{align*}
%     \mcal{F}_A \coloneqq \bigcup_{k=1}^{K-2} \mcal{E}_{M,k},
%     \quad \mcal{F}_B \coloneqq \mcal{E}_K \cup \mcal{E}_a,
% \end{align*}
% where an illustration is shown in Figure~\ref{fig:FA-FB}.
Disjointly partition \(\mcal{F} = \mcal{F}_A \cup \mcal{F}_B\) as in Figure~\ref{fig:FA-FB}.
Also, partition the vertices \([N+1]\) as
\begin{align*}
    \mcal{V}_A \coloneqq \mcal{V}_1 \cup \cdots \cup \mcal{V}_{K-1},
    \quad \mcal{V}_B \coloneqq \mcal{V}_{K} \cup \{N+1\}.
\end{align*}
We first show that \(\mcal{G}\parens{\mcal{V}_A, \mcal{F}_A}\) is chordal, and then construct \(\mcal{G}([N+1], \mcal{F})\) by iteratively adding edges from \(\mcal{F}_B\).

To begin, we claim that \(\mcal{G}\parens{\mcal{V}_A, \mcal{F}_A}\) is chordal with maximal cliques \(\mcal{A}_k \coloneqq \mcal{V}_k \cup \mcal{V}_{k+1}\) for \(k = 1, \ldots, K-2\).
This follows directly from Lemma~\ref{lem:FA-chordal}.

We now construct \(\mcal{G}([N+1], \mcal{F})\) from \(\mcal{G}\parens{\mcal{V}_A, \mcal{F}_A}\) by iteratively adding each vertex in \(\mcal{V}_K \cup \{N+1\}\) and its corresponding edges in \(\mcal{F}_B\).
By the definition of \(\mcal{F}_B\), this iterative method means that each newly added \(v\) will have an edge to all the vertices already present.
Inductively by Lemma~\ref{lem:adding-connected-v}, this implies that \(\mcal{G}([N+1], \mcal{F})\) is chordal with maximal cliques
\(\mcal{C}_k \coloneqq \mcal{A}_k \cup \mcal{V}_k \cup \{N+1\}\)
for \(k = 1, \ldots, K-2\).
% \begin{align*}
%     \mcal{C}_k \coloneqq \mcal{A}_k \cup \mcal{V}_k \cup \{N+1\},
%     \quad k = 1, \ldots, K-2.
% \end{align*}
\end{proof}

\begin{lemma}
\label{lem:FA-chordal}
    The graph \(\mcal{G}\parens{\bigcup_{k=1}^{K-1} \mcal{V}_k, \bigcup_{k=1}^{K-2} \mcal{E}_{M,k}}\)
    % \begin{align*}
    %     \mcal{G}\parens{\bigcup_{k=1}^{K-1} \mcal{V}_k, \bigcup_{k=1}^{K-2} \mcal{E}_{M,k}}
    % \end{align*}
    is chordal with \(K-2\) maximal cliques, where the \(k\)th maximal clique is given by \(\mcal{A}_k \coloneqq \mcal{V}_k \cup \mcal{V}_{k+1}\).
\end{lemma}
\begin{proof}
    % For convenience let us define the function
    % \begin{align*}
    %     \mcal{V}(u) = k,
    %     \enskip \text{such that \(u \in \mcal{V}_k\)},
    % \end{align*}
    % that maps a vertex \(u\) to its corresponding set \(\mcal{V}_k\).
    For convenience we overload notation to define a lookup of vertices to indices: let \(\mcal{V}(u) = k\) iff \(u \in \mcal{V}_k\).
    This mapping is well-defined and unique, since \(\mcal{V}_1, \ldots, \mcal{V}_K\) are a disjoint partition of \([N]\) by construction.
    
    We now show that \(\mcal{G}\) is chordal.
    Let \(u_1, \ldots, u_l, u_1\) be a cycle of length \(l \geq 4\), where let \(u_{l+1} = u_1\).
    Because this is a cycle, there must be some \(u_a, u_b\) that lie in the same \(\mcal{V}_k\), i.e. \(k \coloneqq \mcal{V}(u_a) = \mcal{V} (u_b)\).
    These are then the cases:
    \begin{itemize}    
        \item If \(u_a, u_b\) are not consecutive vertices in the cycle (i.e. neither \(a + 1 \neq b\) nor \(b + 1 \neq a\)), then this means that the edge \((u_a, u_b)\) is a shortcut, since at least one of \(\mcal{E}_{M,k}\) or \(\mcal{E}_{M,k-1}\) is in the edge set of \(\mcal{G}\).

        \item If \(u_a, u_b\) are consecutive vertices (WLOG \(a + 1 = b\)), then this means that \(u_{a-1} \neq u_{b+1}\), since this cycle has length \(\geq 4\).
        Because the structure of each \(\mcal{E}_{M,k} = (\mcal{V}_k \cup \mcal{V}_{k+1})^2\) means that \(u_a, u_b\) share the same neighbors, the edges \((u_{a-1},u_b)\) and \((u_a, u_{b+1})\) therefore exist in \(\mcal{G}\), and are shortcuts for this cycle.
    \end{itemize}
    Because a shortcut always exists for cycles of length \(\geq 4\), conclude that \(\mcal{G}\) is chordal.

    We now analyze the maximal cliques of \(\mcal{G}\).
    Clearly each \(\mcal{A}_k\) is a clique, as its edges are given by \(\mcal{E}_{M,k} = \mcal{A}_k ^2\); we claim that each \(\mcal{A}_k\) is maximal.
    Suppose for some \(\mcal{A}_k\) there is a strictly larger clique \(\mcal{F} \supsetneq \mcal{A}_k\).
    At least one of the two following cases must occur:
    \begin{itemize}
        \item There exists \(j < k\) such that \(\mcal{V}_j \cap \mcal{F} \neq \emptyset\).
        Because \(\mcal{F}\) is a clique, there must be an edge between \(\mcal{V}_j\) and \(\mcal{V}_{k+1}\).
        However this is not possible, since \(\mcal{V}_{k+1}\) only has edges with \(\mcal{V}_k\) and \(\mcal{V}_{k+1}\).
        
        \item There exists \(j > k+1\) such that \(\mcal{V}_j \cap \mcal{F} \neq \emptyset\).
        Because \(\mcal{F}\) is a clique, there must be an edge between \(\mcal{V}_j\) and \(\mcal{V}_k\).
        However this is not possible since \(\mcal{V}_k\) only has edges with \(\mcal{V}_{k-1}\) and \(\mcal{V}_{k+1}\).
    \end{itemize}
    Thus each \(\mcal{A}_k\) is a maximal clique.
    Moreover, there are no more maximal cliques, since each \(\mcal{E}_{M,k}\) exactly corresponds to one \(\mcal{A}_k\).
    Thus, the maximal cliques of \(\mcal{G}\) are exactly \(\mcal{A}_1, \ldots, \mcal{A}_{K-2}\).
\end{proof}

\begin{lemma}
\label{lem:adding-connected-v}
    Let \(\mcal{G}(\mcal{V}, \mcal{E})\) be a chordal graph and \(\mcal{C}_1, \ldots, \mcal{C}_p\) its maximal cliques, and let \(\mcal{G}'(\mcal{V}', \mcal{E}')\) be a supergraph of \(\mcal{G}\) with a new vertex \(v'\) that is connected to every other vertex, i.e. \(\mcal{V}' = \mcal{V} \cup \{v'\}\) and \(\mcal{E}' = \mcal{E} \cup \{(v', v) : v \in \mcal{V}\}\).
    Then \(\mcal{G}'\) is chordal and its maximal cliques are \(\mcal{C}_1 ', \ldots, \mcal{C}_p '\) where each \(\mcal{C}_k ' = \mcal{C}_k \cup \{v'\}\).
\end{lemma}
\begin{proof}
    Let \(u_1, \ldots, u_k, u_1\) be a cycle in \(\mcal{G}'\) with \(k \geq 4\).
    If this cycle lies strictly in \(\mcal{G}\), then there exists a chord because \(\mcal{G}\) is chordal.
    Now suppose that \(u_1 = v'\) without loss of generality, then the edge \((u_i, v') \in \mcal{E}'\) for all vertices of this cycle by construction of \(\mcal{G}'\) --- meaning that there exists a chord.
    This shows that every cycle in \(\mcal{G}'\) of length \(k \geq 4\) has a chord, and thus \(\mcal{G}'\) is chordal.
    
    To show that each \(\mcal{C}_k '\) is a maximal clique of \(\mcal{G}'\), suppose by contradiction that some \(\mcal{C}_k'\) is strictly contained in another clique \(\mcal{C}'\) of \(\mcal{G}'\).
    But then \(\mcal{C}' \setminus \{v'\}\) is a clique in \(\mcal{G}\) which strictly contains some \(\mcal{C}_k = \mcal{C}_k ' \setminus \{v'\}\), which contradicts the maximality assumption of each \(\mcal{C}_k\) in \(\mcal{G}\).
    
    To show that the \(\mcal{C}_k '\) are the only maximal cliques of \(\mcal{G}'\), suppose by contradiction that there exists another maximal clique \(\mcal{C}'\) of \(\mcal{G}'\).
    Now observe that \(\mcal{C}' \setminus \{v'\}\) is a clique in \(\mcal{G}\), and so by assumption it must be contained in some \(\mcal{C}_k\).
    However, this would mean that \(\mcal{C}' \subseteq \mcal{C}_k '\), which contradicts the assumption that \(\mcal{C}'\) is maximal in \(\mcal{G}'\).
\end{proof}

% \subsection{Results for Theorem~\ref{thm:problems-equivalent}}

% \begin{proof}[Proof of Theorem~\ref{thm:problems-equivalent}]
% Since \(Z(\gamma) \in \mbb{S}^{N+1} (\mcal{F})\) and \(\mcal{G}([N+1], \mcal{F})\) is chordal with maximal cliques \(\{\mcal{C}_1, \ldots, \mcal{C}_p\}\), by Lemma~\ref{lem:chordal-psd} we have that \(Z(\gamma) \preceq 0\) iff each \(Z_k \preceq 0\) and
% \(Z(\gamma) = \sum_{k = 1}^{p} E_{\mcal{C}_k}^\top Z_k E_{\mcal{C}_k}\).
% \end{proof}

% \vfill

\section{Proofs of Section~\ref{sec:extensions}}
% \section{Proofs of Section ``Bolstering Chordal-DeepSDP with Double Decomposition and Adjacent-Layer Abstractions''}
\label{appendix:extensions-proofs}

\subsection{Results for Adjacent-Layer Connections}

\begin{proof}[Proof of Lemma~\ref{lem:Qadj}]
% We proceed by rewriting \eqref{eq:Qk-sum} into the form of Lemma~\ref{lem:Qadj}.
Rewrite \eqref{eq:Qk-sum} into the Lemma~\ref{lem:Qadj} form:
\begin{align*}
    &\sum_{k = 1}^{K-1}
        \begin{bmatrix} \\ \star \\ \\ \end{bmatrix}^\top
        Q_k
        \begin{bmatrix}
            W_k & 0 & b_k \\
            0 & I_{n_{k+1}} & 0 \\
            0 & 0 & 1
        \end{bmatrix}
        \begin{bmatrix} E_k \\ E_{k+1} \\ E_a \end{bmatrix}
        \mbf{z} \\
    &= \sum_{k = 1}^{K-1}
        \begin{bmatrix} \\ \star \\ \\ \end{bmatrix}^\top
        Q_k
        \begin{bmatrix} W_k x_k + b_k \\ x_{k+1} \\ 1 \end{bmatrix} \\
    &= \sum_{k=1}^{K-1}
        \begin{bmatrix} \\ \star \\ \\ \end{bmatrix}^\top
        P_k ^\top Q_k P_k
        \begin{bmatrix}
            W_1 x_1 + b_1 \\
            \vdots \\
            W_{K-1} x_{K-1} + b_{K-1} \\
            x_2 \\
            \vdots \\
            x_K \\
            1
        \end{bmatrix} \\
    &= \sum_{k = 1}^{K-1} 
        \begin{bmatrix} \\ \star \\ \\ \end{bmatrix}^\top
        P_k ^\top Q_k P_k
        \begin{bmatrix}
            A & b \\ B & 0 \\ 0 & 1
        \end{bmatrix}
        \mbf{z} \\
    &= \mbf{z}^\top
        \begin{bmatrix}
            A & b \\ B & 0 \\ 0 & 1
        \end{bmatrix}^\top
        \underbrace{
        \parens{\sum_{k = 1}^{K-1} P_k ^\top Q_k P_k}
        }_{Q_{\mrm{adj}}}
        \begin{bmatrix}
            A & b \\ B & 0 \\ 0 & 1
        \end{bmatrix}
        \mbf{z}
\end{align*}
where \(P_k\) is a projection acting on \(\mrm{vcat}(A\mbf{x} + b, B\mbf{x}, 1)\) that selects out the vector \(\mrm{vcat}(W_k x_k + b_k, x_{k+1}, 1)\).
\end{proof}

\begin{proof}[Proof of Theorem~\ref{thm:new-Z-sparsity}]
It suffices to show that each term of the sum~\eqref{eq:Qk-sum} has sparsity \(\mcal{E}\).
Consider the expression
\begin{align*}
    \begin{bmatrix} \\ \star \\ \\ \end{bmatrix}^\top
    Q_k
    \begin{bmatrix}
        W_k & 0 & b_k \\
        0 & I_{n_{k+1}} & 0 \\
        0 & 0 & 1
    \end{bmatrix}
    \begin{bmatrix} E_k \\ E_{k+1} \\ E_a \end{bmatrix},
\end{align*}
whose \((i,j)\) entry is dense iff exactly one of the following disjoint conditions holds:
\begin{enumerate}
    \item[(1)] \(i \in \mcal{V}_k\) and \(j \in \mcal{V}_k\)
    
    \item[(2)] \(i \in \mcal{V}_k\) and \(j \in \mcal{V}_{k+1}\) (or the symmetric case)
    
    \item[(3)] \(i \in \mcal{V}_k\) and \(j = N + 1\) (or the symmetric case)
    
    \item[(4)] \(i \in \mcal{V}_{k+1}\) and \(j \in \mcal{V}_{k+1}\)
    
    \item[(5)] \(i \in \mcal{V}_{k+1}\) and \(j = N + 1\) (or the symmetric case)
    
    \item[(6)] \(i = N + 1\) and \(j = N + 1\)
\end{enumerate}
If (1), (2), or (4) holds, then \((i,j) \in \mcal{E}_{M, k}\);
if (3), (5), or (6) holds, then \((i,j) \in \mcal{E}_{a}\).
Thus the \((i,j)\) entry is dense implies that \((i,j) \in \mcal{E}_M \cup \mcal{E}_a\), and so \((i,j) \in \mcal{E}\).
\end{proof}

\end{document}